\documentclass{article}

\usepackage{microtype}
\usepackage{graphicx}
\usepackage{paralist}
\usepackage{booktabs}

\PassOptionsToPackage{sort,numbers}{natbib}

\usepackage[colorlinks]{hyperref}

\usepackage[preprint]{neurips_2026}

\usepackage{amsmath}
\usepackage{amssymb}
\usepackage{amsthm}
\usepackage{paralist}

\usepackage[english]{babel}

\usepackage[dvipsnames]{xcolor}
\usepackage{enumitem}
\usepackage{tcolorbox}
\usepackage{colortbl}
\usepackage{apptools}
\usepackage{titletoc}

\usepackage{thmtools}
\usepackage{thm-restate}
\usepackage{appendix}

\declaretheorem[name=Theorem]{theorem}

\declaretheorem[name=Proposition, sibling=theorem]{proposition}
\declaretheorem[name=Lemma, sibling=theorem]{lemma}
\declaretheorem[name=Corollary, sibling=theorem]{corollary}

\declaretheorem[name=Definition, style=definition]{definition}
\declaretheorem[name=Remark, style=remark]{remark}

\AtBeginEnvironment{appendices}{\crefalias{section}{appendix}\crefalias{subsection}{appendix}}

\usepackage{amsmath,amsfonts,bm}

\def\eqref#1{equation~\ref{#1}}

\def\1{\bm{1}}

\DeclareMathAlphabet{\mathsfit}{\encodingdefault}{\sfdefault}{m}{sl}
\SetMathAlphabet{\mathsfit}{bold}{\encodingdefault}{\sfdefault}{bx}{n}

\usepackage{todonotes}
\usepackage{bbm}

\usepackage{soul}
\usepackage{bm}

\usepackage{framed}
\usepackage{nicefrac}

\renewenvironment{leftbar}[1][\hsize]
{\MakeFramed{\hsize#1\advance\hsize-\width\FrameRestore}}
{\endMakeFramed}
\usepackage[capitalize, nameinlink]{cleveref}

\crefname{appendix}{appendix}{appendices}
\Crefname{appendix}{Appendix}{Appendices}

\newcommand{\ms}[1]{\left\{\!\!\left\{#1\right\}\!\!\right\}}

\title{On the Rademacher Complexity of Graph Neural Networks: Unifying Expressivity and Geometry}

\author{Martin Carrasco$^*$ \\
University of Fribourg \\
\texttt{martin.carrrascocastaneda@unifr.ch}  \\
    \And 
    Caio F. Deberaldini Netto$^*$\\
Johns Hopkins University\\
    \texttt{cnetto1@jh.edu} \\
    \And
    Vahan A. Martirosyan$^*$ \\
    Université Paris-Saclay \\
    \texttt{vahan.martirosyan@centralesupelec.fr}
    \And
    Ehimare Okoyomon$^*$\\
Technical University of Munich \\
    \texttt{e.okoyomon@tum.de}  \\
    \And 
Caterina Graziani \\
University of Siena\\
\texttt{caterina.graziani2@unisi.it} \\
}

\begin{document}

\maketitle

\begin{abstract}
	Understanding the interplay between generalization, expressivity, and the geometry of the input space is a central challenge in graph learning. The expressivity of Graph Neural Networks (GNNs) is typically characterized through their correspondence with graph invariants, such as those from the Weisfeiler-Leman (WL) hierarchy. While more expressive GNNs can distinguish a richer set of graphs, they are also associated with weaker generalization guarantees. Previous works have addressed this trade-off using the VC dimension, a purely combinatorial measure, independent of the training data. In this work, we adopt a data-dependent measure of generalization, the empirical Rademacher complexity, and derive tight generalization bounds that jointly consider the expressive power of GNNs and the geometry of the underlying input space. Specifically, any graph invariant that upper-bounds a GNN's expressive power partitions the input space into equivalence classes, and we show that the empirical Rademacher complexity is controlled by the distribution of training samples across these classes. Moving beyond discrete partitions, we incorporate the geometry of the input space and derive covering-number bounds under Lipschitz continuity, showing that the complexity cost can be mitigated when the hypothesis class remains smooth over the data geometry. In addition, we prove that the empirical Rademacher complexity is Lipschitz continuous with respect to the Wasserstein distance between empirical measures supported on different datasets. This yields robustness and generalization guarantees under sampling variability. Importantly, our framework is not restricted to message-passing GNNs or WL, but extends to arbitrary GNN architectures and their associated invariants, providing a step toward a unified theory of GNN generalization.
\end{abstract}
\section{Introduction}

Graph Neural Networks (GNNs) \citep{scarselli2008graph,gilmer2020message} have achieved remarkable success across domains ranging from social networks and knowledge graphs to computational chemistry \citep{zhou2020graph}. This empirical progress has motivated an extensive theoretical effort to characterize their \emph{expressive power}, most commonly measured by the ability to distinguish non-isomorphic graphs
\citep{morris2019weisfeiler,xu2018powerful,huang2021short}. A now-classical line of work
establishes a tight correspondence between GNN architectures and the Weisfeiler--Leman (WL)
hierarchy of graph isomorphism tests: message-passing GNNs (MPGNNs) are at most as expressive as the $1$-dimensional WL test ($1$-WL) \citep{xu2018powerful,JMLR:v24:22-0240},
and analogous correspondences have since been established for a broad spectrum of more expressive variants. Expressivity, however, speaks only to what a model \emph{can} represent on the training distribution; it is largely silent on the arguably more fundamental question of how well it \emph{generalizes} to unseen graphs \citep{vasileiou2025survey}.

A recent line of work has begun to bridge this gap by connecting expressivity to classical
generalization theory. \citet{morris2023wl} established the first direct link between GNN
expressivity and VC dimension, showing that for GNNs with piecewise-linear activations the
VC dimension is controlled by the number of graphs distinguishable by $1$-WL, yielding
bounds of the form $\mathcal{O}(P \log(puP))$, where $P$ is the number of parameters, $u$
the number of node colors, and $p$ the number of activation pieces. \citet{dinverno2025vcgnn}
extended this analysis to Pfaffian activations such as $\tanh$ and \textit{sigmoid}, with bounds
likewise governed by the maximum number of node colors per graph. Further related connections between expressivity and generalization are discussed in
\Cref{app:additional_RW}.

An orthogonal and more data-dependent approach studies the Rademacher complexity of GNN
hypothesis classes. \citet{garg2020generalization} provided the first of such bounds for
MPGNNs, carefully exploiting their local permutation invariance. While tighter than
VC-based guarantees, these bounds depend solely on architectural parameters and do not
expose how expressivity---as captured by WL-style colorings---influences the complexity.
A complementary line of work has advanced a \emph{metric} view of GNN expressivity,
replacing discrete WL equivalence with pseudo-metrics on graphs. \citet{chuang2022tree}
introduce the Tree Mover's Distance (TMD), a hierarchical optimal-transport pseudo-metric
under which MPNNs are provably Lipschitz, yielding stability and generalization bounds;
\citet{chen2022weisfeiler} define a polynomial-time Weisfeiler--Leman distance that
quantitatively refines the $1$-WL equivalence and bounds the Gromov--Wasserstein
distance; and \citet{boker2023fine} extend $1$-WL and MPNNs continuously to graphons,
characterizing their separating power through a tree pseudo-metric and thereby turning
the binary WL-distinguishability question into a quantitative one. Closer to our setup,
\citet{levie2023graphon} and \citet{maskey2022generalization} leverage graphon-based
metrics---in particular a cut distance under which MPNNs are Lipschitz---to derive
generalization bounds via covering numbers.
Most recently, \citet{maskey2026graph} formalize the empirical observation that
generalization hinges on \emph{structure--label correlation}: they introduce a family of
$\zeta$-Tree Mover's Distances ($\zeta$-TMDs), parametrized by an arbitrary graph
invariant $\zeta$ (e.g., degree distributions or $k$-WL colors) that encodes a prescribed
level of expressivity, and study a regime in which labels correlate with a fixed
$\zeta$-TMD. Within this framework, they derive data-dependent bounds that decompose the
generalization gap into a capacity term and a structural-similarity term measured by the
$\zeta$-TMD between training and test graphs, showing that greater expressivity improves
generalization only when it tightens the structure--label alignment.
Despite this progress, a unified account that simultaneously (i) ties the empirical
Rademacher complexity of GNNs to the equivalence classes induced by an \emph{arbitrary}
expressivity-bounding graph invariant, (ii) refines this combinatorial picture through
the geometry of the input space under arbitrary pseudo-metrics via covering numbers, and
(iii) provides stability guarantees under sampling variability, has so far been missing.

Our work directly addresses the gap identified above by relating the expressive power of
GNNs, as captured by an arbitrary graph coloring invariant, to their empirical
Rademacher complexity, and by extending the analysis to the underlying input geometry.
We thereby provide a principled explanation for the observed tension between expressivity
and generalization, and for why this tension is often milder in practice than worst-case
analyses suggest. Our main \emph{contributions} are:

\paragraph{(i) Expressivity-aware complexity bounds.} We derive tight bounds on
the empirical Rademacher complexity of GNNs governed by the number of equivalence
classes induced by an arbitrary expressivity-bounding graph invariant (e.g., the WL
hierarchy), together with how the training sample is distributed across these classes. Moreover, we refine the classical Dudley entropy integral, collapsing the effective dimensionality of the output space to the number of distinguishable graphs.

\paragraph{(ii) Geometric refinement via covering numbers.} Moving beyond discrete
partitions, we show that under Lipschitz continuity with respect to an arbitrary
pseudo-metric on graphs, the Rademacher complexity admits data-dependent bounds in
terms of covering numbers of the input space, formalizing how smoothness of the
hypothesis class over the data geometry can offset the cost of high expressivity.

\paragraph{(iii) Stability under sampling variability.} We prove that the empirical
Rademacher complexity is Lipschitz continuous with respect to the Wasserstein
distance between empirical dataset measures, yielding generalization guarantees under
resampling and distribution shift of the training set.

Our framework applies to arbitrary GNN architectures and their associated invariants,
going beyond the standard message-passing / $1$-WL pairing and offering a unified view of
the interplay between expressivity, generalization, and data geometry.

\section{Notation and preliminaries}\label{sec:preliminaries}
For $n \geq 1$, let $[n] := \{1, 2, \dots, n\}$. We use $\ms \ldots$ to denote multisets, i.e., the generalization of sets allowing for multiple instances of each of their elements.
\paragraph{Graphs.} A graph $G=(V, E)$ is a pair with finite set of vertices or nodes $V$ and edges \mbox{$E\subseteq\{\{u, v\} \subseteq V\,\vert\, u \neq v\}$}. For ease of notation, we denote the edge $\{u, v\}$ in $E$ by $(u, v)$ or $(v, u)$.  If not otherwise stated, we set $n := |V|$, and the graph is of order $n$. Let $\mathcal{N}(v)$ be the \emph{neighborhood} of a node $v \in \mathcal{V}$, i.e. the set of all nodes adjacent to $v$, and $d(v)$ the \emph{degree} of a node $v \in \mathcal{V}$, i.e., the number of neighbors $|\mathcal{N}(v)|$. An \textbf{attributed graph} $G =(V, E, \alpha)$ is a triple with a graph $(V,E)$ and node-attribute function $\alpha: V \to A$, where $A$ is a a finite subset of $\mathbb{R}^{d}$, for some $d > 0$. We consider the space of finite, simple, undirected, attributed graphs, denoted by $\mathcal{G}$.

\paragraph{Metric spaces.} A \emph{pseudo-metric space} is an ordered pair $(\mathcal{X}, \mathrm{d})$ where $\mathcal{X}$ is a set and \mbox{$\mathrm{d} \colon \mathcal{X} \times \mathcal{X} \to \mathbb{R}_{\geq 0}$} is a function satisfying the following properties, for any $x,y,z \in \mathcal{X}$:  
\begin{inparaenum}[\bfseries I.]
    \item \mbox{$\mathrm{d}(x,x) = 0$},
    \item $\mathrm{d}(x,y) = \mathrm{d}(y,x)$,
    \item $\mathrm{d}(x,y) \leq \mathrm{d}(x,z) + \mathrm{d}(y,z)$.
\end{inparaenum}
Additionally, if \mbox{$\mathrm{d}(x,y) =0 \implies x =y$}, then the pair $(\mathcal{X}, \mathrm{d})$ is a \emph{metric space}.
For $x\in\mathcal{X}$, we denote the \emph{closed ball} of radius $\varepsilon$ centered on $x$ as $B_{\varepsilon}(x) = \{ y \mid \mathrm{d}(x, y) \leq \varepsilon \}$.\\ 
A set $\{x_1,\ldots,x_n\}\subseteq \mathcal{X}$ is an $\varepsilon-$cover of $K\subseteq\mathcal{X}$ if $K\subseteq\bigcup\limits_{i=1}^n B_{\varepsilon}(x_i)$.
The \emph{covering number} of $K$ is the minimal number of closed balls needed to \emph{cover} $K$. We denote it as
\begin{equation}
    \mathcal{N}(K, \mathrm{d}, \varepsilon) :=\min \{ n \mid \text{ there exists an } \varepsilon\text{-cover of cardinality } n \}.
\end{equation}
Note that the balls of the cover do not need to be disjoint, i.e. points can be in the intersection of different balls. However, the covering number provides a direct way to construct a partition of $K$ as $\{C_1, \dots, C_p\}$ in the following way:
\begin{equation*}
    C_i = \{ x \in K \mid \mathrm{d}(x, x_i) = \min_{j \in [p]} \mathrm{d}(x, x_j) \},\quad  \text{with $p=\mathcal{N}(K, \mathrm{d}, \varepsilon)$.}
\end{equation*}
To break the ties, we take the smallest $i$ in the above. We treat the space of all graphs $\mathcal{G}$ as a pseudometric-space, assuming the existence of a pseudo-metric which lets us measure the distance between elements.

\paragraph{Continuity on metric spaces.} Let $(\mathcal{X}, \mathrm{d}_{\mathcal{X}})$ and $(\mathcal{Y}, \mathrm{d}_{\mathcal{Y}})$ be two pseudometric spaces. A function $f\colon \mathcal{X} \to \mathcal{Y}$ is $\lambda$-Lipschitz continuous if for any $x,x' \in \mathcal{X}$
\begin{equation}
    \mathrm{d}_{\mathcal{Y}} (f(x), f(x')) \leq \lambda \cdot \mathrm{d}_{\mathcal{X}}(x, x').
\end{equation}
Let $\mathcal{F} := \{f_\theta \mid{\theta} \in \Theta\}$ be a parametrized family of functions $f_{\theta} \colon \mathcal{X} \to \mathbb{R}$ such that $f_{\theta}$ is Lipschitz with constant $\lambda_{\theta} < \infty$. We say that $\mathcal{F}$ is Lipschitz with \emph{certificate} $\lambda<\infty,$ if $\lambda$ is an upper-bound on the set of Lipschitz constants such that $\forall \theta \in \Theta: \lambda_{\theta} \leq \lambda$.

\paragraph{Message passing graph neural networks.}\label{prelim:gnn}
Message-passing GNNs (MPGNNs) learn real-valued vectors, called \emph{embeddings}, for each node by iteratively updating their features based on aggregated information from their neighbors. The embedding $h^{(\ell)}(v)$ for node $v$ at layer $\ell$ is computed as:
\begin{equation}\label{eq:gnn}
    h^{(\ell)}(v) = \texttt{COMBINE}^{(\ell)}\left( h^{(\ell-1)}(v), \texttt{AGGREGATE}^{(\ell)}\left(\ms{h^{(\ell-1)}(u)}_{ u \in \mathcal{N}(v)}  \right)\right).
\end{equation}
where $\texttt{AGGREGATE}^{(\ell)}$ and $\texttt{COMBINE}^{(\ell)}$ are differentiable parameterized functions, e.g. neural networks, and $\texttt{AGGREGATE}^{(\ell)}$ is  permutation invariant over multisets.
After $L$ layers, we obtain the final node embedding, which we denote as $h^{(L)}(v)$. For graph-level tasks, these are aggregated into a graph representation $h^{(L)}(G)=\texttt{READOUT}(\ms{ h_{L}(v) }_{v\in V}),$
where $\texttt{READOUT}$ is a differentiable parameterized function.

\paragraph{Expressivity-bounding graph invariant.} \label{prelim:expressivity_bounding_graph_invariant}
{ 
A \emph{graph invariant} is a function $ T: \mathcal{G} \to \mathcal{C} $ that assigns a label to each graph such that if $G$ and $H$ are isomorphic, it holds that $ T(G) = T(H) $. Any such invariant induces and equivalence relation $\equiv_T$ on $\mathcal{G}$ such that equivalent graphs are assigned the same label.
The expressive power of MPGNNs is often studied via their capability to distinguish non-isomorphic graphs. To test whether two graphs are isomorphic, a graph invariant function is applied to both graphs. If the values of the two graphs differ, the graphs are non-isomorphic. If the values are the same, the algorithm is inconclusive, meaning that the two graphs may be, but are \emph{not guaranteed to be}, isomorphic. 
}

Let $\mathcal A$ be a GNN architecture with associated hypothesis class $\mathcal{F_A}$. We say that a graph invariant $T$ upper bounds the expressive power of $\mathcal A$, denoted by $\mathcal A\preceq   T$, if for all $G,H \in \mathcal G,$ 
 \mbox{$T(G)=T(H)  \implies  f(G)=f(H)$},
for every $f \in \mathcal F_{\mathcal A}$.
Equivalently, $\mathcal A$ cannot distinguish graphs that are identified by $T$.

\paragraph{1-Weisfeiler-Leman.}
Research has shown that MPNNs are at most as powerful as the 1-Weisfeiler-Leman (1-WL) test, a well-known isomorphism heuristic \cite{xu2018powerful, morris2019weisfeiler, wl68}.  {Thus, the 1-WL is an expressivity-bounding graph invariant for MPNNs. Throughout the paper, we adopt this correspondence between MPGNNs and 1-WL as a running example to support the exposition and enhance readability. Nonetheless, our results hold in general for arbitrary GNN architectures along with the graph invariant that upper bounds their expressive power.}

The 1-WL test  is a \emph{color refinement algorithm}, namely a mapping that assigns to each graph $G=(V,E)$ a function $ c^{\mathrm{WL}}: V \to \mathcal{C}$, where $\mathcal{C}$ is a set of node labels, or colors. It works by partitioning the nodes of $G$ into equivalence classes, where equivalent nodes are assigned the same color based on their neighborhood structure. At each iteration $\ell$, the color 
{$c^{(\ell)}(v)$ }
of a node $v$ is updated by hashing its previous color with the multiset of its neighbors' colors:
{ 
\begin{equation}\label{eq:wl-update-fn}
     c^{(\ell)}(v)= \text{\texttt{HASH}}\Bigl(c^{(\ell-1)}(v), \ms{c^{(\ell-1)}(u)}_{ u \in \mathcal{N}(v)}\Bigr).
\end{equation}
}This process continues until the partitioning is stable. We denote by 
{ 
$c^{\mathrm{WL}}(v)$}
the color of node $v$ at convergence of the partitions. That is, 
{ 
$c^{\mathrm{WL}}(v):=c^{(L)}(v)$ }
where $L$ is the first iteration after which the partition no longer changes.
We define the \emph{color} of graph $G=(V,E)$ as the multiset of colors of its nodes:
\begin{equation}
    T^{\mathrm{WL}}(G) = \ms{c^{\mathrm{WL}}(v)}_{v\in V}.
\end{equation}

Every color refinement algorithm like $1-$WL induces a graph invariant $T^{\mathrm{WL}}$ acting over the graph $G$ by aggregating the colors of its nodes into a multiset. More concretely, given two graphs $G=(V,E)$ and $G'=(V',E')$, we can define the equivalence relation $\overset{\mathrm{WL}}\equiv$ induced by $T^\mathrm{WL}$ on graphs as: 
$$G \; \overset{\mathrm{WL}}\equiv \; G' \Leftrightarrow T^{\mathrm{WL}}(G) = T^{\mathrm{WL}}(G') \Leftrightarrow \ms{c^{\mathrm{WL}}(v)}_{v\in V}=\ms{c^{\mathrm{WL}}(u)}_{u\in V'}.$$
\paragraph{Generalization.} We briefly review the main definitions in the theory of Rademacher complexity. We invite the reader to consult \citet{Mohri2012-nv} for a comprehensive treatment of the topic. Let $S=\{(G_i,y_i)\}_{i\in [m]}\sim \mathcal{D}^m $ be a dataset composed of $m$ i.i.d. samples which we assume are drawn from an underlying distribution $\mathcal{D}$ on $\mathcal{G}\times \mathcal{Y}$, with $\mathcal{Y}=\{-1,+1\}$. Sometimes we subsume the set $\mathcal{Y}$, writing $S= \{G_1, \dots G_m\}$. 
For any fixed GNN architecture $\mathcal{A}$, we denote by
$\mathcal{F_A}$ the class of possible {graph-level} functions that can be learned by this GNN.

Given a loss function $\ell$ that measures the prediction error, we define for each $f\in\mathcal{F_A}$ the empirical and true (or population) risk, respectively, by
\begin{equation}
{L}_{S}(f) = \frac{1}{m}\sum_{j=1}^{m}\ell (f(G_j),y_j), \quad L(f)= \mathbb{E}_{(G,y)\sim\mathcal{D}}[\ell(f(G),y)].
\end{equation}
The generalization gap is defined as the difference between the true and empirical risk, denoted by $e^{gen}(f, S) := L(f) - L_S(f)$, and it is bounded by the model complexity. To quantify complexity, let $\sigma_{1},\ldots, \sigma_{m}$ be independent Rademacher variables and define the \emph{empirical} Rademacher complexity
\begin{equation}
{\mathcal{R}}_{S}(\mathcal{F_A}) = \mathbb{E}_{\sigma} \left[ \sup_{f\in\mathcal{F_A}} \frac{1}{m}\sum_{j=1}^m \sigma_j f_\theta(G_{j})\right],
\end{equation}
with \emph{population} counterpart 
\begin{equation}
\mathcal{R}_{m}(\mathcal{F_A}) = \mathbb{E}_{S\sim \mathcal{D}^{m}}\left[{\mathcal{R}}_{S}(\mathcal{F_A})\right].
\end{equation}

Rademacher complexity measures how well a function class $\mathcal{F_A}$ can correlate with random noise. High Rademacher complexity indicates that there exists a function in $\mathcal{F_A}$ that is potentially "overfitting" the labels.  

\section{Generalization bounds: From equivalence classes to metrics}For a fixed GNN architecture $\mathcal A$, let $\Theta_{\mathcal A}$
denote the set of admissible parameter values and define its hypothesis class as
$$\mathcal F_{\mathcal A}=\{f_\theta:\mathcal G\to[-1,1]\mid\theta\in\Theta_{\mathcal A}\}.$$
Here, $f_\theta(G)$ denotes the graph-level output of the GNN with parameters
$\theta$ on input graph $G$. We extend here the definition of MPGNNs provided in \Cref{prelim:gnn} to describe how the graph-level output $f(G)$ is computed, thereby specifying the hypothesis class $\mathcal{F_A}$ under consideration.
Let $h^{(L)}(G)\in \mathbb{R}^d$ be the global embedding of the graph $G$, obtained by combining the node embeddings using a \texttt{READOUT} function such as \textit{sum}, \textit{max}, or \textit{mean}. Then, the GNN output $f_\theta(G)$ is computed by applying an activation function (e.g., tanh) $\psi(\cdot)$
to the linearly transformed graph embedding:\begin{equation}f_\theta(G) = \psi(\bm{\beta}^\intercal h_L(G)) \in [-1,1],\end{equation}
where $\bm\beta \in \mathbb{R}^d$ is a trainable parameter.

The generalization gap of the class $\mathcal{F_A}$ can be controlled by its empirical Rademacher complexity on the training sample $S$, through a standard learning theory argument.
Specifically, combining the
Rademacher generalization bound of \cite[Theorem 3.3]{Mohri2012-nv} with
Talagrand's contraction lemma \cite{ledoux2011probability}  we obtain the following bound for any loss $\ell:[-1,1]^2\to\mathbb R$ that is
$\gamma$-Lipschitz with respect to function output.

\begin{restatable}{proposition}{radLossBound}
Let $\ell$ be a Lipschitz loss function, of constant $\gamma$.
	For any $\delta > 0$, with probability at least $1 - \delta$, the following holds for all $f \in \mathcal{F_A}$:
	\begin{equation}
		L(f) \le {L}_{S}(f)  +
		2\gamma\, {\mathcal{R}}_{S}(\mathcal{F_A})
		+3\sqrt{\frac{\ln(2/\delta)}{2m}}.
	\end{equation}
\end{restatable}
Some examples of such Lipschitz loss functions in the context of graph classification are the logistic loss ($\log$ \textit{loss}), the cross-entropy (CE) when applied to the output of a \textit{softmax} layer \cite{mao2023cross} or to the output of a logistic function when its input is bounded (check   \cref{app:bound_l_ce_lipsc}), and a margin loss \cite{garg2020generalization}. Particularly for this work, since our hypothesis class is $\mathcal{F_A}=\{f_\theta \,\colon\mathcal{G}\to[-1,1] \mid \theta\in\Theta_{\mathcal{A}}\}$, we can either use directly the latter loss with an activation function that gives outputs in the interval $[-1, 1]$ (e.g., \textit{tanh}), or combine the GNN's output with a linear transformation $[-1, 1] \to [0, 1]$ and use one of the other two loss functions (see \cref{app:bound_gen_l_ce_lipsc}).

\subsection{Rademacher complexity through the lens of graph invariants}\label{sec:rc-meets-colors}
In this section we focus our attention on the connection between expressivity and generalization. While GNN expressivity is traditionally tied to the WL test, our framework applies more broadly: the application of any graph invariant $T$ on $\mathcal{G}$ has a direct relationship with the Rademacher complexity of the hypothesis class realized by GNNs bounded in expressivity by $T$. To begin, we introduce how a graph invariant acts on a set of graphs. Applying $T:\mathcal{G}\to \mathcal{C}$ on a sample \mbox{$S=\{G_1, \dots, G_m\}\subset\mathcal{G}^m$} results in a partition of $S$ in $p$ equivalence classes $C_1, \dots, C_p$. Each $C_j:=\{G\in S: T(G)=c_j\}$ contains all graphs assigned the same label $c_j\in \mathcal{C}$.  For intuition, if $T$ represents the standard 1-WL algorithm, the labels $c_j$ correspond to graph colors. We define the empirical distribution induced by $T$ on $S$ as:
\[
	\mu^T_S:\mathcal{C}\to[0,1], \qquad
\mu^T_S(c_j):=\frac{|C_j|}{|S|}.
\]
When $S$ and $T$ are clear from the context, we simply write
$\mu_j:=\mu^T_S(c_j)$, and interpret $\mu_j$ as the proportion of graphs sharing the same label $c_j$ in $S$.

	{
		This partition imposes structural constraints on the function class $\mathcal{F_A}$, since $\mathcal{A} \preceq T$ (see \Cref{prelim:expressivity_bounding_graph_invariant}).
	}
Consequently, any function implementable by the architecture must be constant over the equivalence classes $C_j$, which limits its possibility to overfit arbitrary labels, since not all labelings are compatible with the partitioning.
In the following, we bound the Rademacher complexity of a function class $\mathcal{F_A}$ in terms of the empirical distribution $\mu_S^T$ induced by $T$ on $S$. While the upper bound on this complexity follows directly from the structural constraint $\mathcal{A} \preceq T$, which enforces constancy within each equivalence class $C_j$, to derive a matching lower bound we require that the hypothesis class retains complete flexibility in assigning arbitrary labels to each distinct partition. We formalize this theoretical upper bound of flexibility as follows.

\begin{definition}[Maximally expressive]
	Let $T: \mathcal{G} \to \mathcal{C}$ be a graph invariant and let $S$ be a finite sample partitioned by $T$ into $p$ equivalence classes $C_1, \dots, C_p$. Let $\mathcal{A}\preceq T$. We say that $\mathcal{F}_{\mathcal{A}}$ is \emph{maximally expressive} if it contains functions that can perfectly realize any arbitrary labeling over the quotient space induced by $T$. Formally, for any sequence of labels $y \in \{-1, 1\}^p$, there exists a function $f \in \mathcal{F}_{\mathcal{A}}$ such that:
	\begin{equation}
		f(G) = y_j \quad \forall G \in C_j, \quad \forall j \in [p].
	\end{equation}
\end{definition}
Building on this definition, the following proposition establishes that the empirical Rademacher complexity is controlled by the distribution of graphs in the equivalence classes, and is tightly bounded when the hypothesis class is maximally expressive.

\begin{restatable}[Bounds based on partitions]{proposition}{upcolor}\label{prop:upper-lower_bound}
	Let $S = \{G_1, \dots, G_m\}$ be a sample of $m$ graphs, partitioned into $p$ disjoint sets $C_1, \dots, C_p$ by a graph invariant $T$.
	Let $\mathcal{F_A}$ be a class of functions $f_\theta:\mathcal{G}\to[-1,1]$ such that $\mathcal{A} \preceq  T$. Then, the empirical Rademacher complexity of $\mathcal{F_A}$ on $\mathcal{S}$ is upper bounded by:
	\begin{equation*}
		\mathcal{R}_{S}(\mathcal{F_A}) \leq {\sum\limits_{j=1}^p  \sqrt\frac{\mu_j}{m}},
	\end{equation*}
	where $\mu_j:=\frac{|C_j|}{m}$.

	Moreover, if $\mathcal{F_A}$ is \emph{maximally-expressive}, then the empirical Rademacher complexity of $\mathcal{F_A}$ is tightly bounded from both sides\begin{equation*}
		\sum\limits_{j=1}^p\sqrt{\frac{\mu_j}{2m}} \leq \mathcal{R}_{S}(\mathcal{F_A}) \leq {\sum\limits_{j=1}^p \sqrt\frac{\mu_j}{m}}.
	\end{equation*}
\end{restatable}

The proof of \Cref{prop:upper-lower_bound} can be found in  \Cref{app:proof_color}.
\Cref{prop:upper-lower_bound} shows that the bound on the Rademacher complexity is tight if the class contains maximally-expressive functions. In this case, the empirical Rademacher complexity admits both upper and lower bounds of the same order, i.e., $\mathcal{R}_S(\mathcal{F_A})=\mathcal O\left(\sum_j\sqrt{{\mu_j}/{m}}\right)$. In the case where $\mathcal{F_A}$ \textit{does not} contain maximally-expressive functions, the upper bound still holds. We provide a more general upper bound for classes of real-valued functions $f:\mathcal{G}\to\mathbb{R}$ in  \Cref{app:real-valued_functions_bound}.

This bound reveals three key insights regarding how the distribution of equivalence classes and the sample size impact generalization. First, for a fixed number of equivalence classes, $p$, the bound is highly sensitive to the class frequencies $\mu_j$. Because the square root function is strictly concave, the upper bound is maximized when the dataset is uniformly distributed across all classes (i.e., $\mu_j = 1/p$ for all $j$). In this uniform regime, the upper bound reduces to $\sqrt{p/m}$. Conversely, if the dataset is highly disproportionate (i.e. dominated by a few dense equivalence classes) the Rademacher complexity is significantly reduced, leading to tighter generalization guarantees.

Second, the bound explicitly links generalization error to the structural heterogeneity of the dataset. Assuming a uniform distribution, the complexity scales directly with $\sqrt{p}$. Therefore, a highly heterogeneous dataset with a massive number of distinct equivalence classes (or a highly expressive function class $\mathcal{F_A}$ that induces them) will incur a larger generalization penalty than a more homogeneous one.

Finally, the dependence on $1/\sqrt{m}$ matches standard learning-theoretic intuition: increasing the sample size tightens the bound regardless of architectural expressivity.

We apply the same intuition to improve the  Dudley entropy integral bound (see \cref{app:dudley}) on the empirical
Rademacher complexity \citep{bartlett2017spectrally}) by incorporating the number of graph classes in the inequality. Let $\mathcal{N}(\mathcal{F_A}_{|S},\varepsilon, ||\cdot||_{2})$ be the covering number of $\mathcal{F_A}_{|S}$ \footnote{The notation $\mathcal{F_A}_{|S}$ denotes the \emph{restriction} of the class of functions $\mathcal{F_A}$ to functions with domain in $S$.} at radius $\varepsilon$ under $ ||\cdot||_{2}$ (the $\ell_2$ norm).  Then:
\begin{restatable}[Covering using Dudley's entropy]{proposition}{dudleycoveringA}
	\label{prop:dudley_improved}
Let $S = \{G_1, \dots, G_m\}$ be a sample of $m$ graphs, partitioned into $p$ disjoint sets $C_1, \dots, C_p$ by a graph invariant. Let $\mathcal{F_A}$ be a class of functions $f_\theta: \mathcal{G} \to [-1, 1]$ whose output is the same on each graph of a fixed $C_j$ (i.e. $\mathcal{A}\preceq T$). Assume $\bm{0} \in \mathcal{F_A}$. The empirical Rademacher complexity of $\mathcal{F_A}$ on $S$ is bounded by:
	\begin{equation} \mathcal{R}_S(\mathcal{F_A}) \le \inf_{\alpha>0}\left(\frac{4\alpha\sqrt{p}}{m}+\frac{12}{m}\int_{\alpha}^{\sqrt{m}}\sqrt{\log\mathcal{N}(\mathcal{F_A}_{|S},\varepsilon, ||\cdot||_{2})}d\varepsilon\right) \end{equation}
	where the bound is reduced due to the $p$-dimensional structure of the output space $\mathcal{F_A}_{|S}$.
\end{restatable}

The proof can be found in ~\Cref{app:proof-dudley} and is very similar to the proof of Lemma A.5 from \cite{bartlett2017spectrally}.
{Relative to prior work \cite{bartlett2017spectrally}, the first term in the bound is tighten from $1/\sqrt{m}$ to $\sqrt{p}/m$, yielding a concrete improvement and an explicit characterization in terms of graph colors.} Furthermore, this bound applies generally to all GNN architectures but can be further refined for a specific function class by bounding the covering number of $\mathcal{F_A}_{|S}$. We refer the interested reader to \citet{garg2020generalization} for an example of such a covering number bound for message-passing GNNs.

\subsection{A geometric view}

While the previous section established generalization bounds based on discrete partitions induced by graph invariants, this framework is inherently limited by its categorical nature. The Weisfeiler-Lehman (WL) algorithm, while a foundational paradigm for GNN expressivity, typically yields a binary classification: two graphs are either equivalent under the refinement process or they are not. This discrete perspective does not naturally account for the \emph{degree} of structural similarity between non-isomorphic graphs, nor does it easily reconcile structural comparisons across graphs of varying sizes. To address these limitations and provide a more nuanced characterization of GNN capacity, we transition our analysis from discrete equivalence classes to a geometric perspective. By considering the geometry of the input space through the lens of pseudo-metrics, we can capture finer-grained structural relationships that discrete partitions overlook. This shift follows recent trends in the literature that utilize continuous measures of similarity, such as the \emph{Tree Mover's Distance} (TMD)\citep{chuang2022tree}, to analyze the stability and generalization of GNNs \citep{vasileiou2024covered}.

In this section, we extend our Rademacher analysis to arbitrary pseudo-metric spaces with respect to which our hypothesis class is Lipschitz continuous, as defined in \Cref{sec:preliminaries}. For example MPGNNs are Lipschitz continuous with respect to standard WL-equivalent metrics, such as the TMD \cite{vasileiou2024covered, sverdlov2025fswgnn}.

In the discrete regime, the invariant imposes an \emph{exact} structural bottleneck (i.e. identical outputs for indistinguishable graphs). In the continuous regime, this hard constraint is replaced by a smoothness condition, the Lipschitzianity of the class, realizing the alignment between the model's output and the input geometry. In cases where there is a strong correlation between the geometry of the input space and the task, alignment of the hypothesis class is relevant for bounding the generalization error \citep{maskey2026graph}. \\

\begin{remark}
	Note that by the assumption of Lipschitz continuity of $\mathcal{F_A}$ we implicitly enforce the constraint that if $T$ distinguished two graphs,  $\mathrm{d}$ must be \textit{at least as expressive} as $T$ since $|f_\theta(G)-f_\theta(G')|>0$ implies that  $\mathrm{d}(G,G')>0$.
\end{remark}

To obtain an upper bound in terms of the geometry of the input only the Lipschitz assumption is \emph{sufficient}. However, to achieve tightness and a lower bound, a stronger agreement criteria between $\mathrm{d}$ and $\mathcal{F_A}$ is needed. We formalize this agreement through the definition $T$-equivalent metric, a generalization of a WL-equivalent metric, as formulated by \citet{sverdlov2025fswgnn}.

\begin{definition}[$T$-equivalent metric]\label{def:T-equivalent}
	Let $T$ be a graph invariant. A pseudo-metric $\mathrm{d}:\mathcal{G}\times \mathcal{G}\to \mathbb{R}_{\geq 0}$ is $T$-equivalent if for all pairs $G,G'\in \mathcal{G}$,   $\mathrm{d}(G,G')=0 \iff G\equiv_T G'$.
\end{definition}

\begin{restatable}[Bound on metric space of input]{proposition}{upradmetricB}
	Let $(\mathcal{G}, \mathrm{d})$ be a pseudo-metric space of graphs $\mathcal{G}$ and $S \subset \mathcal{G}$ be a finite subset. Let $\mathcal{F_A}$ be a Lipschitz class of parametrized functions  \mbox{$f_\theta \,\colon \mathcal{G}\to [-1, 1] $}.Let \mbox{$p_\varepsilon:=\mathcal{N}(S,\varepsilon, \mathrm{d}_{\mathcal{G}})$} denote the covering number of $S$ with radius $\varepsilon$, with respect to the metric $\mathrm{d}$.
	Then, the empirical Rademacher complexity of $\mathcal{F_A}$ on ${S}$ is upper bounded by
	\begin{equation}
		\mathcal{R}_{S}(\mathcal{F_A}) \leq  \sum_{j=1}^{p_{\varepsilon}} \sqrt{\frac{\mu_j^{\varepsilon}}{m}} + \varepsilon \lambda
	\end{equation}
where  $\lambda$ is the Lipschitz certificate of $\mathcal{F_A}$ and
$\mu_j^{\varepsilon}$ is the proportion of the graphs in the $j$th ball of the $\varepsilon$-covering of $S$. \\
	Furthermore, if $\mathrm{d}$ is a $T$-equivalent pseudo-metric and $\mathcal{F_A}$ is maximally expressive, the complexity is tightly bounded from both sides:
	\begin{equation}
		- \varepsilon \lambda +\sum_{j=1}^{p_{\varepsilon}} \sqrt{\frac{\mu_j^{\varepsilon}}{2m}} \leq \mathcal{R}_{S}(\mathcal{F_A}) \leq  \sum_{j=1}^{p_{\varepsilon}} \sqrt{\frac{\mu_j^{\varepsilon}}{m}} + \varepsilon \lambda
	\end{equation}
\end{restatable}
One term in the bound corresponds to the Rademacher complexity obtained by approximating $S$ with its $\varepsilon$-cover, while the other term $\varepsilon\lambda$ accounts for the approximation error incurred by replacing each sample point with its cover representative. This geometric formulation naturally subsumes the discrete case. If the space is equipped with the discrete metric induced by the invariant (i.e., $\mathrm{d_T}(G, G') = 0$ if $G \equiv_T G'$, and $1$ otherwise), for any $\varepsilon<1$ the $\varepsilon$-balls collapse exactly onto the equivalence classes of $T$. In this regime, the approximation error vanishes, the covering number $p_\epsilon$ matches the number of partitions $p$, and we perfectly recover the exact discrete bounds from
\cref{prop:upper-lower_bound}.

\textcolor{blue}{markdown comparison table, commented}

\paragraph{Recovering tightness via quotient spaces.}The bounds presented above are tight when the pseudo-metric $\mathrm{d}$ is aligned with the invariant $T$ bounding the architecture. If $d$ can distinguish graphs that the invariant T identifies, namely if $\mathrm{d}(G,G')>0$ while $T(G)=T(G')$, then the covering number $\mathcal N(S,\varepsilon,d)$ may overestimate the effective complexity of the hypothesis class. In this case, we can project the metric space onto the quotient space $S/{\equiv_T}$ induced by the invariant.
Equipping this quotient space with a valid pseudo-metric $\bar d$ allows the complexity bound to depend on the covering number of the quotient space, $\mathcal N(S/{\equiv_T},\varepsilon,\bar d)$, rather than on that of the original sample. This removes distinctions that are invisible to the architecture and can therefore yield a sharper bound. Note that this quotient formulation represents the most general version of our framework: it naturally subsumes the discrete partitioning and the $T-$equivalent metric cases, where $T-$equivalent graphs are already collapsed to distance zero, rendering the quotient operation implicit. We defer the formal construction and the corresponding bounds to \Cref{app:quotient}.

\section{Generalization and stability across datasets}\label{sec:stability}

The previous results bound the Rademacher complexity of a class of functions $\mathcal{F_A}$ on a fixed sample $S$. We now turn to investigate the \emph{stability} of GNNs across different samples, approaching this problem through the lens of optimal transport. We show that the Rademacher complexity remains meaningful under noisy perturbations, and in particular that $\mathcal{R}_{S}(\mathcal{F_A})$ is Lipschitz‐continuous with respect to the $1-$Wasserstein distance between the two datasets.
The specific formulation of this distance depends on the underlying metric structure of the graph space. Following the structure of the previous section, we present two variants: one operating on the discrete empirical distributions over equivalence classes, and another based on continuous pseudo-metrics.

\begin{restatable}[Stability of the Rademacher]
{theorem}{stability}\label{thm: stability}
Let $S=\{G_1, \cdots, G_m\}$ and $S'=\{G'_1, \cdots, G'_m\}$ be two  graph samples drawn from $\mathcal{D}^m$. Let  $W_1(S,S')$ denote the 1-Wasserstein distance between the two sets and let $\alpha\geq 0$. Then, the empirical Rademacher complexity  is $\alpha-$Lipschitz continuous with respect to $W_1$:
    \begin{equation}
        \vert \mathcal{R}_{S}(\mathcal{F_A}) - \mathcal{R}_{S'}(\mathcal{F_A}) \vert {\leq \alpha \cdot W_1(S,S').}
    \end{equation}
    
The scaling factor $\alpha$ and the distance $W_1$ are defined according to the underlying space:
    \begin{itemize}
\item Discrete partitioning regime:\\
        $ \alpha=2$ and $W_1(S,S')=W_1(\mu^T_{S},\mu^T_{S'}):=\frac{1}{2}\sum\limits_{c\in \mathcal{C}} \vert \mu^T_{S}(c) - \mu^T_{S'}(c)\vert.$
        \item Continuous pseudo-metric space:\\    $\alpha=\lambda$ and $W_1(S,S')=W_1(\nu_S,\nu_{S'}):=\min\limits_{\pi\in\mathfrak S_m}\frac{1}{m}\sum\limits_{i=1}^m \mathrm{d}(G_i,G_{\pi(i)}'),$\\
     where $\mathfrak S_m$ denotes the set of all permutations over $[m]$ and $\nu_S:\mathcal{G}\to [0,1]$, is the uniform empirical measure $\nu_S(G)=\frac{1}{m}\sum\limits_{i=1}^m \delta_{G_i}(G)$ supported on the sample $S$. 
    \end{itemize}
\end{restatable}
In the first case, $W_1$ represents the optimal transport between the empirical distributions of graphs into the equivalence classes, which reduces to the scaled Total Variation Distance (TVD) \citep{verdu2014total} between the discrete measures $\mu^T_S$ and $\mu^T_{S'}$. In the second scenario, $W_1$ represents the optimal transport between the uniform empirical measures $\nu_S$ and $\nu_{S'}$, with respect to the pseudo-metric $\mathrm{d}$. The proof can be found in \Cref{app:proof-stability}.

\Cref{thm: stability} allows us to directly quantify the generalization error on a target $S'$ if we know the complexity of a reference sample, $S$, and its distance to the target, $W_1(S,S')$. 
This formally shows how structural differences between samples impact the reliability of the model.
\begin{restatable}[Cross-dataset generalization bound]{corollary}{stabilitygeneralization} 
{For any $\delta > 0$, with probability at least $1 - \delta$, the following holds for all $f \in \mathcal{F_A}$ and any $\gamma-$Lipschitz loss function:
    $$e^{gen}(f, S')\leq 
    2\gamma\left[\mathcal{R}_{S}(\mathcal{F}) 
    +\alpha W_1(S,S')\right]+3\sqrt{\frac{\ln(2/\delta)}{2m}}$$}
 \end{restatable}
The proof is available in \Cref{app:proof-cross-dataset}.

Furthermore, the continuity property of $\mathcal{R}(\mathcal{F_A})$ naturally yields robustness guarantees under localized noise. If the target dataset $S'$ is not drawn independently but is a perturbed version of $S$, the Rademacher bounds remain tightly controlled.

\begin{definition}[$\delta$-perturbation] Let $S=\{G_1, \cdots, G_m\}$ and $S'=\{G'_1, \cdots, G'_m\}$ two sets of graphs. We say that $S'$ is a $\delta-$perturbation of $S$ with respect to a pseudo-metric $d$ if 
there exists a bijection $f\colon S \to S'$, such that $\forall \ i \in [m]$, $f(G_i) = G_i'$ and
$d(G_i,G_i')\leq\delta$.
\end{definition}
If $S'$ is a $\delta-$perturbation of $S$, their distance $W_1(S,S')$ is equal to $\delta,$ and we can bound the Rademacher on $S'$ using the complexity computed on $S$.
\begin{corollary}[Robustness of the Rademacher under perturbations]
    Let $\mathcal{F_A}$ be a Lipschitz functions class with certificate $\lambda$. Let $S$ be a sample and $S'$ a $\delta-$perturbation of $S$. Then, the empirical Rademacher complexity of $\mathcal F$ on $S'$ is bounded by:
    \begin{align*}
R_{S'}(\mathcal{F_A}) \leq R_S(\mathcal{F_A})  + \lambda\delta.
    \end{align*}
\end{corollary}
 \section{Conclusion}\label{sec:conclusion}

In this study, we establish a principled connection between GNN generalization theory, expressivity, and the metric geometry of the underlying graph space. By employing empirical Rademacher complexity, we move beyond static, combinatorial measures to provide a data-dependent characterization of GNNs capacity.
This work generalizes notions from previous studies \cite{vasileiou2025survey, maskey2026graph} to arbitrary pseudo-metric spaces for which the hypothesis class is Lipschitz continuous.
We first formulate generalization bounds from expressivity-bounding graph invariants, before refining these bounds directly using the geometry of the underlying space. We show tightness when the pseudo-metric is in agreement with the invariant that bounds the hypothesis class and $\mathcal{F_A}$ is maximally expressive. In other cases, we can nonetheless ensure tightness by projecting $S$ to the quotient space induced by the invariant and repeating the procedure with an adequate metric there. Moreover, we show that our results are stable and robust under sampling variability.

\paragraph{Limitations \& Future work.}\label{sec:limitations-FW}
We identify several limitations of our study that offer a roadmap for future extensions.
Firstly, although our framework offers a unified perspective on the expressivity-generalization trade-off, the current analysis is restricted to binary classification and does not yet extend to multi-class evaluation or regression tasks.
Given known extensions of Rademacher complexity to multi-class classification and bounded regression problems \citep{li2018multi, Mohri2012-nv}, we anticipate adapting those techniques to broaden the theoretical scope of our results. Furthermore, as this work is primarily focused on establishing a principled theoretical foundation, we do not provide empirical verification of the derived bounds. While we formally characterize how the distribution of equivalence classes and the geometry of the input space control generalization, a systematic empirical study remains a necessary next step. Future work should investigate how these Rademacher-based bounds apply in practice across diverse GNN architectures, specifically verifying how the number of distinct partitions and the smoothness of the hypothesis class over the data geometry influence observed generalization error.
Finally, our current results bound the empirical Rademacher complexity without making assumptions on the underlying graph distribution. A promising avenue for future research involves characterizing the behavior of the \emph{true} Rademacher complexity when graphs are sampled from known probabilistic models, such as Random Graph Models (RGMs). In this context, graphons \cite{lovasz2012large} offer a natural framework for such evaluations; investigating how graph size influences color diversity in graphon samples would provide valuable insights into the asymptotic behavior of our theory and its implications for large-scale graph learning.

\newpage

\bibliography{reference}
\bibliographystyle{abbrvnat}

\clearpage
\newpage

\begin{appendices}

\titlecontents{section}
	[0pt]
	{}
	{\thecontentslabel\quad}
	{}
	{\titlerule*[0.5pc]{.}\contentspage}
	\titlecontents{subsection}
	[1.5em]
	{}
	{\thecontentslabel\quad}
	{}
	{\titlerule*[0.5pc]{.}\contentspage}

	\startcontents
	\printcontents{}{1}{{\large\textbf{Appendix~(Supplementary Materials)}}}

	\vskip15pt
	\hrule
	\vskip5pt

\section{Additional related works}\label{app:additional_RW}

The fundamental trade-off between expressivity and generalization in GNNs is attracting increasing attention within the community. A recent work by \citet{maskey2026graph} demonstrated that more expressive GNNs may have worse generalization capabilities, unless their increased complexity is balanced by sufficiently large training sets or reduced structural distance between training and test graphs. Their analysis introduces pseudo-metrics that capture structural similarity and reveal when expressivity hurts generalization. The case when more expressive power affects generalization is further refined by \citet{franks2024weisfeiler}. The authors propose using \emph{partial concepts} to derive bounds of VC dimension independent of the length of the embedding vector, $d$. Additionally, they show that for certain classes of graphs there are tighter lower bounds, thus confirming that more expressivity is not always worse. {  In a more general approach, \citep{rauchwerger2024generalization} establishes a bound on the generalization error independent of both the data and the parameter of \emph{any} MPNN. However, this approach eliminates the nuances that exist in different datasets. Depending on the task at hand, a more fine-grained analysis requires taking into account the distribution of graphs, for instance, by using the construction of graphons such as in \citep{levie2023graphon}. Nonetheless, that approach has its own drawbacks, since under the employed metrics, sparse graphs converge to the empty graph, which hinders seamless adoption in our context. 
Taking an alternative route, \citet{vasileiou2024covered} propose a unified framework that jointly analyzes generalization, robustness, and expressivity.
This relies on a new pseudo-metric termed \emph{Forest Distance}, inspired by \emph{Tree Distance} \citep{boker2021graph}. Nevertheless, the bounds are not data dependent and while vertex-attributed graphs are considered, only discrete attributes are assumed. Additionally, it only works if the aggregation method for graphs is mean pooling.} 
Meanwhile, \citet{li2024towards} proposed the notion of a k-variance margin-based generalization bound, defining the structural quality of graph embeddings in terms of their expressive power.

Collectively, these findings align with our work, illustrating a more nuanced relationship between model expressivity and generalization, and they are not restricted to message-passing GNNs, as is also the case for our study. Our work differs in that it provides a theoretical analysis grounded in Rademacher complexity, using coloring-based partitioning as a formal lens to characterize expressivity.

\section{Covering numbers and Dudley's integral}\label{app:metric}
In this appendix, we introduce some theoretical tools used throughout the paper and in the proofs. Our goal is to keep the presentation self‐contained, highlighting only the results needed in later proofs. In particular, we recall basic notions on covering numbers and we include the Dudley entropy integral, which is instrumental in deriving bounds on the Rademacher complexity.

\subsection{Dudley’s entropy integral}

Dudley's entropy integral bounds the Rademacher complexity of a function class using its covering numbers. We consider the restriction of $\mathcal{F_A}$ to a sample $S = \{x_1, x_2, \dots, x_m\}$, which is the set of vectors: $\mathcal{F_A}_{|S}=\{(f_\theta(x_1),\dots ,f_\theta(x_m)) | f_\theta \in \mathcal{F_A}\} \subseteq  \mathbb{R}^{m}$. The covering number in the theorem is computed with respect to the standard Euclidean norm $\vert\vert\cdot\vert\vert_{2}$ on $\mathbb{R}^{m}$.

\begin{theorem}[Dudley's entropy integral bound \cite{bartlett2017spectrally}]
\label[theorem]{app:dudley}
Let $\mathcal{F_A}$ be a real-valued function class taking values in $[0, 1]$ , and assume that $\bm{0} \in \mathcal{F_A}$. Then
\begin{equation}
\label{eq-dudley-bound}
\mathcal{R}_{S}(\mathcal{F_A}) \leq \inf_{\alpha > 0}\left\{\frac{4\alpha}{\sqrt{m}} + \frac{12}{m}\int_{\alpha}^{\sqrt{m}}\sqrt{\log\mathcal{N}(\mathcal{F_A}_{|S},\epsilon,\vert\vert\cdot\vert\vert_{2}})\mathrm{d}\epsilon\right\}.
\end{equation}
\end{theorem}

In  \cref{prop:dudley_improved}, we derive a tighter bound using the $ p$-dimensional structure of $\mathcal{F_A}_{|S}$.

\section{Proofs }\label{app:proof}
\subsection{Lipschitz requirements }\label{app:loss}
\begin{proposition}\label{app:bound_l_ce_lipsc} Let $\ell_{CE}$ be the cross-entropy loss function. Moreover, let $f_\theta(G_i) = \psi(\bm{\beta}^\top\phi(G_i))$ be a GNN output, where $\psi(\cdot)$ is the logistic activation function (i.e., $\psi\colon \mathbb{R} \to [0, 1]$), and $\phi(\cdot)$ the GNN's final representation. Assume that $\|\phi\|_\infty \leq b_\phi$ and $\|\bm{\beta}\|_1 \leq B_\beta$, for constants $b, B_\beta > 0$. Therefore, $\ell_{CE}$ is \textit{Lipschitz continuous}.
\end{proposition}
\begin{proof}
    The cross-entropy loss is defined as:
    \begin{align*}
        \ell_{CE}(f_\theta(G), y) = -\sum_{i=1}^{m} \left[ y_i \log(f_\theta(G_i)) + (1-y_i)\log(1-f_\theta(G_i)) \right].
    \end{align*}
    
The partial derivative of $\ell_{CE}$ with respect to $z_i = \bm{\beta}^T\phi(G_i)$ is
    \begin{align}
        \frac{\partial \ell_{CE}}{\partial z_i}
        &= -\left[ y_i \frac{\psi'(z_i)}{\psi(z_i)} - (1-y_i)\frac{\psi'(z_i)}{1-\psi(z_i)} \right] \nonumber\\
        &= -\left[y_i \frac{\psi(z_i)(1-\psi(z_i))}{\psi(z_i)} - (1-y_i)\frac{\psi(z_i)(1-\psi(z_i))}{1-\psi(z_i)} \right] \nonumber\\
        &= -\left[ y_i (1-\psi(z_i)) - (1-y_i)\psi(z_i) \right] \nonumber\\
        &= \psi(z_i) - y_i  \hspace{25em} (^*) \nonumber
    \end{align}

Since $\phi(G_i)$ is bounded by $b_\phi$ in $L_\infty$-norm, and $\|\bm\beta\|_1 \leq B_\beta$, we have
    \begin{align*}
        |z_i| = |\bm\beta^T \phi(G_i)| \leq \sum_{j=1}^d |\beta_j| |\phi_j(G_i)| \leq \|\phi\|_\infty \|\bm\beta\|_1 \leq b_\phi B_\beta.
    \end{align*}

Thus, $z_i \in [-b_\phi B_\beta, b_\phi B_\beta]$, and the sigmoid function $\psi(z)$ satisfies:
\[
\psi(-b_\phi B_\beta) \leq \psi(z_i) \leq \psi(b_\phi B_\beta),
\]
for all $G_i$. 

Then, we have that
\[
(^*): |\psi(z_i) - y_i| \leq \max\{ |\psi(b_\phi B_\beta)|, |1-\psi(b_\phi B_\beta)| \},
\]
since $y_i \in \{0,1\}$.

The derivative of the loss with respect to $z_i$ is bounded; therefore, $\ell_{CE}(f_\theta(G_i), y_i)$ is Lipschitz continuous in $\phi(G_i)$.

\end{proof}

\begin{proposition}\label{app:bound_gen_l_ce_lipsc}
    Assume the conditions from Prop.~\ref{app:bound_l_ce_lipsc} hold. Moreover, assume that the activation function is a sigmoid, i.e., $\psi\colon \mathbb{R} \to [a, b]$, for $a, b \in \mathbb{R}$ and $a < b$. In addition, assume that its derivative is bounded, $|\psi'(x)| \leq C$, for $C > 0$. Analogously, let $\ell_{CE}$ be the cross-entropy loss function, and define $g\colon [a, b] \to [0, 1]$, $g(x) = \tfrac{x - a}{b-a}$. Therefore, $\ell_{CE}(g\circ f_\theta(G_i))$ is Lipschitz continuous.
\end{proposition}
\begin{proof}
    Using the cross-entropy loss definition shown before, we have:

    \begin{align*}
        \ell_{CE}(g\circ f_\theta(G), y) = -\sum_{i=1}^{m} \left[ y_i \log(f_\theta(g\circ G_i)) + (1-y_i)\log(1-g\circ f_\theta(G_i)) \right].
    \end{align*}

    Analogously, the partial derivative of $\ell_{CE}$ with respect to $z_i = \beta^\top \phi(G_i)$ is

    \begin{align}
        \frac{\partial \ell_{CE}}{\partial z_i}
        &= -\left[ y_i \frac{1}{g(\psi(z_i))} - (1-y_i)\frac{1}{1-g(\psi(z_i))} \right]\cdot\frac{1}{g'(\psi(z_i))}\cdot\psi'(z_i) \nonumber\\
        &= -\left[ y_i \frac{1}{g(\psi(z_i))} - (1-y_i)\frac{1}{1-g(\psi(z_i))} \right]\cdot\frac{1}{(b-a)}\cdot\psi'(z_i) \nonumber\\
        &= \left[\frac{g(\psi(z_i)) - y_i}{g(\psi(z_i))(1 - g(\psi(z_i)))} \right]\cdot\frac{C}{(b-a)} \hspace{13.5em} (^*) \nonumber
    \end{align}

    Since $|z_i| \leq b_\phi B_\beta$, then $a < \psi(-b_\phi B_\beta) \leq \psi(z_i) \leq \psi(b_\phi B_\beta) < b$. Hence, we have that $0 < g(-b_\phi B_\beta) \leq g(\psi(z_i)) \leq g(b_\phi B_\beta) < 1$, for all $G_i$, and
    \[
    (^*): \left|\frac{g(\psi(z_i))- y_i}{g(\psi(z_i))(1 - g(\psi(z_i)))} \right|\cdot\frac{C}{(b-a)} \leq \frac{C}{(b - a)}\cdot \max\biggl\{\frac{1}{|g(\psi(b_\phi B_\beta))|}, \frac{1}{|1-g(\psi(b_\phi B_\beta))|} \biggr\}.
    \]

    Again, because the derivative of the loss with respect to $z_i$ is bounded, we have that $\ell_{CE}(g\circ f_\theta(G_i), y_i)$ is Lipschitz continuous in $\phi(G_i)$.
\end{proof}

\subsection{Upper bounds for real-valued functions}\label{app:real-valued_functions_bound}
\begin{proposition}
{Let $S = \{G_1, \dots, G_m\}$ be a sample of $m$ graphs, partitioned into $p$ disjoint classes $C_1, \dots, C_p$ by a graph invariant. Let $\mathcal{F}:=\{f_\theta:\mathcal{G}\to \mathbb{R} \mid \theta\in\Theta\}$ be a class of functions whose output $f_\theta(G)$ is the same on each graph of a fixed class $G \in C_j$.}
The empirical Rademacher complexity of $\mathcal{F}$ on $S$ is bounded by:
\begin{align}
    \mathcal{R}_{S}(\mathcal{F}) \leq \frac{\sup_{\Theta} \|f_{\theta}\|_2\sqrt{p}}{m} \label{eq:real_norm2}\\ \quad \quad \text{and} \quad \quad \mathcal{R}_{S}(\mathcal{F})\leq\frac{\sup \|f_{\theta}\|_\infty}{m} \sum_{j=1}^p \sqrt{|I_j|}\label{eq:real_norminfty}
\end{align} 
where $\|f_{\theta}\|_2$ and $\|f_{\theta}\|_\infty$ are respectively the $\ell_2-$norm and the $\ell_\infty-$ norm  of the function's outputs over the sample $S$.
\end{proposition}

\paragraph{Proof of \Cref{eq:real_norm2}.}
\begin{proof}
 The proof proceeds by first reorganizing the sum by equivalence classes, then applying the Cauchy-Schwarz inequality to separate the function-dependent norm from the Rademacher variables, and finally using Jensen's inequality.

Applying $T$ on a sample \mbox{$S=\{G_1, \dots, G_m\}$} results in $p$ partitions $C_1, \dots, C_p$, where each $C_j=\{G\in S: T(G)=c_j\}$ is an equivalence class containing all graphs with the same color $c_j$. Since $\mathcal{A} \preceq  T$, let $f_\theta(G_j)$ be the constant output for any graph in partition $C_j$ and let $I_j$ be the set of indices of graphs in $C_j$: 
    $$
        I_j := \{\, i \in [m] : T(G_i) = c_j \,\}.
    $$
    
    First, we write the definition of the empirical Rademacher complexity and group the sum over the $p$ partitions.
    
    \begin{align}\label{eq:radcolor_setup}
        \mathcal{R}_{S}(\mathcal{F_A}) &= \mathbb{E}_{\bm{\sigma}} \left[ \sup_{\Theta}\frac{1}{m} \sum_{i=1}^m \sigma_i f_\theta(G_i) \right] \nonumber \\
        &= \frac{1}{m} \mathbb{E}_{\bm{\sigma}} \left[ \sup_{\Theta} \sum_{j=1}^p f_j(\Theta) \sum_{i \in I_j} \sigma_i \right] 
    \end{align}
Let $f_j(\Theta)$ be the constant output for any graph in partition $I_j$.

Let $Z_j = \sum_{i \in I_j} \sigma_i$. The inner sum is $\sum_{j=1}^p f_j(\Theta) Z_j$. We apply the Cauchy-Schwarz inequality to this sum over $j$:

$$ \sum_{j=1}^p f_j(\Theta) Z_j = \sum_{j=1}^p\Big(f_j(\Theta)\sqrt{|I_j|}\Big)\Big(\frac{Z_j}{\sqrt{|I_j|}}\Big) \leq \sqrt{\sum_{j=1}^p f_j(\Theta)^2 |I_j|} \cdot \sqrt{\sum_{j=1}^p \frac{Z_j^2}{|I_j|}} $$

The first term on the right is precisely the L2-norm $L(\Theta)$, since $\sum_{j=1}^p f_j(\Theta)^2 |I_j| = \sum_{i=1}^m f_j(\Theta)^2 = L(\Theta)^2$. Substituting this back into the main expression gives:
\begin{align*}
\mathcal{R}_{S}(\mathcal{F}) &\leq \frac{1}{m} \mathbb{E}_{\bm{\sigma}} \left[ \sup_{\Theta} \left( L(\Theta) \cdot \sqrt{\sum_{j=1}^p \frac{Z_j^2}{|I_j|}} \right) \right] \\
&= \frac{1}{m} \mathbb{E}_{\bm{\sigma}} \left[ \left(\sup_{\Theta} L(\Theta)\right) \cdot \sqrt{\sum_{j=1}^p \frac{Z_j^2}{|I_j|}} \right] \\
&= \frac{\sup_{\Theta} L(\Theta)}{m} \mathbb{E}_{\bm{\sigma}} \left[ \sqrt{\sum_{j=1}^p \frac{(\sum_{i \in I_j} \sigma_i)^2}{|I_j|}} \right]
\end{align*}
The second line follows because the term involving the Rademacher variables $\sigma_i$ does not depend on $\Theta$, allowing us to separate the supremum. The third line follows because $\sup_{\Theta} L(\Theta)$ is a constant with respect to the expectation over $\sigma$.

Next, we apply Jensen's inequality to the expectation. Since the square root function is concave, $\mathbb{E}[\sqrt{X}] \leq \sqrt{\mathbb{E}[X]}$.
$$ \mathcal{R}_{S}(\mathcal{F}) \leq \frac{\sup_{\Theta} L(\Theta)}{m} \sqrt{ \mathbb{E}_{\bm{\sigma}} \left[ \sum_{j=1}^p \frac{(\sum_{i \in I_j} \sigma_i)^2}{|I_j|} \right] } $$
Finally, we evaluate the expectation inside the square root. By the linearity of expectation and the fact that $\sigma_i$ are independent random variables with $\mathbb{E}[\sigma_i]=0$ and $\mathbb{E}[\sigma_i^2]=1$, we have:
\begin{align*}
\mathbb{E}_{\bm{\sigma}} \left[ \sum_{j=1}^p \frac{(\sum_{i \in I_j} \sigma_i)^2}{|I_j|} \right] &= \sum_{j=1}^p \frac{\mathbb{E}_{\bm{\sigma}}[(\sum_{i \in I_j} \sigma_i)^2]}{|I_j|} \\
&= \sum_{j=1}^p \frac{\sum_{i\in I_j}\mathbb{E}[\sigma_i^2]+\sum_{i\neq k}\mathbb{E}[\sigma_i\sigma_k]}{|I_j|} \\
&= \sum_{j=1}^p \frac{\sum_{i\in I_j}\mathbb{E}[\sigma_i^2]}{|I_j|} \\
&= \sum_{j=1}^p \frac{|I_j|}{|I_j|} \\
&= p. 
\end{align*}
Substituting this result back gives the final bound:
$$ \mathcal{R}_{S}(\mathcal{F}) \leq \frac{\sup_{\Theta} L(\Theta)\sqrt{p}}{m} = \frac{\sup_{\Theta} \|f_\theta \|_2\sqrt{p}}{m}.$$
\end{proof}

\paragraph{Proof of \Cref{eq:real_norminfty}}

\begin{proof}
Let $Z_j = \sum_{i \in I_j} \sigma_i$. The inner sum is $\sum_{j=1}^p f_\theta(G_j) Z_j$. We apply the Cauchy-Schwarz inequality to this sum over $j$:
    $$
        \sum_{j=1}^p f_\theta(G_j) Z_j \leq \sum_{j=1}^p |f_\theta(G_j)Z_j|
    $$
By Hölder inequality, we have that $ \sum_{j=1}^p |f_\theta(G_j)Z_j|\leq \|f_\theta\|_\infty\|Z\|_1.$ Then,
\begin{align}
        \mathcal{R}_{S}(\mathcal{F_A}) &\leq \frac{1}{m} \mathbb{E}_{\bm{\sigma}} \left[   \sup_{\Theta} \left( \|f_{\theta}\|_\infty \cdot \|Z\|_1 \right) \right]  \nonumber \\
        &= \frac{1}{m} \mathbb{E}_{\bm{\sigma}} \left[   \left( \sup_{\Theta}  \|f_{\theta}\|_\infty \right) \cdot \sum_{j=1}^p |Z_j| \right]  \nonumber \\
        &= \frac{\sup_{\Theta} \|f_{\theta}\|_\infty}{m} \mathbb{E}_{\bm{\sigma}} \left[     \sum_{j=1}^p |Z_j|  \right]  \nonumber 
    \end{align}
    The second line follows because the term involving the Rademacher variables $\sigma_i$ does not depend on $\Theta$, allowing us to separate the supremum. The third line follows because $\sup_{\Theta}$ is a constant with respect to the expectation over $\sigma$.
\begin{align*}
\mathbb{E}_{\bm{\sigma}} \left[     \sum_{j=1}^p |Z_j|\right] 
            &= \sum_{j=1}^{p} \mathbb{E}_{\bm{\sigma}} \left[   \left| Z_j   \right| \right] \\
        &\leq \sum_{j=1}^{p} \sqrt{\mathbb{E}_{\bm{\sigma}}\left(\sum_{i \in I_j} \sigma_i\right)^2} \\
        &=  \sum_{j=1}^{p} \sqrt{|I_j|}
\end{align*}

    Substituting back in the main equation we get
\begin{align}
    \mathcal{R}_{S}(\mathcal{F_A}) &\leq
    \frac{\sup_{\Theta} \|f_{\theta}\|_\infty}{m} \mathbb{E}_{\bm{\sigma}} \left[     \sum_{j=1}^p |Z_j|  \right] \\
    &\leq \frac{\sup_{\Theta} \|f_{\theta}\|_\infty}{m} \sum_{j=1}^{p_{\varepsilon}} \sqrt{|I_j|}.
\end{align}

\end{proof}
\subsection{Bounds: From discrete partitions to metric}
\upcolor* \label{app:proof_color}
\begin{proof}
Let $T : \mathcal{G} \to \mathcal
C$ be a graph invariant that induces a partition of $S$ into $p$ disjoint equivalence classes $C_1,\dots,C_p$. Let $c_j$ be the unique value assigned by $T$ to all graphs belonging to class $C_j$. We set
\begin{equation}\label{eq:partition-multiplicity}
I_j := \{\, i \in [m] : T(G_i) = c_j \,\}.
\end{equation}
Since the architecture $\mathcal{A}$ is as good as $T$ in distinguishing graphs, the output of $f_\theta \in \mathcal{F_A}$ is assumed to be identical for all graphs in the same partition, i.e. $f_\theta(G_i)=f_\theta(G_k)$ for any $i,k \in I_j$. 
Let $f_j(\Theta)$ be the constant output for any graph in partition $I_j$.
First, we write the definition of the empirical Rademacher complexity and group the sum over the $p$ partitions.
\begin{align}\label{eq:grouping}
\mathcal{R}_{S}
&= \mathbb{E}_{\sigma} \left[ \sup_{\Theta}\frac{1}{m} \sum_{i=1}^m \sigma_i f_\theta(G_i) \right] \nonumber\\
&= \frac{1}{m} \mathbb{E}_{\sigma} \left[ \sup_{\Theta} \sum_{j=1}^p f_j(\Theta) \sum_{i \in I_j} \sigma_i \right] 
\end{align}
The sum over $m$ individual graphs is reorganized into a sum over the $p$ partitions ($I_j$), where each partition contains graphs with the same output. This allows us to factor out the constant function value, represented by $f_j(\Theta)$. \paragraph{Upper bound.}
\begin{align*}
    \frac{1}{m} \mathbb{E}_{\sigma} \left[ \sup_{\Theta} \sum_{j=1}^p f_j(\Theta) \sum_{i \in I_j} \sigma_i \right]
&\leq \frac{1}{m} \mathbb{E}_{\sigma} \left[\sum_{j=1}^p \left|\sum_{i \in I_j} \sigma_i\right| \right] \\
&= \frac{1}{m} \sum_{j=1}^p \mathbb{E}_{\sigma} \left[\left|\sum_{i \in I_j} \sigma_i\right|\right]  \\
&\leq \frac{1}{m} \sum_{j=1}^p \sqrt{\mu_j}
\end{align*}
The first inequality follows from the fact that $f_j(\Theta)\in[-1,1]$ and the maximum will be when $f_j(\Theta) = \mathrm{sign}(\sum_{i \in I_j} \sigma_i)$.
The second inequality applies Khintchine's inequality \cite{haagerup1981best}, which bounds the expected value of the absolute value of a sum of Rademacher variables.
\paragraph{Lower bound.}
We begin our proof of the lower bound from equation \ref{eq:grouping}. By hypothesis, the function class contains functions as expressive as $T$, i.e. as good as $T$ in distinguishing graphs. Moreover, if $\mathcal{A}$ is a universal approximator over the classes, there exist a $f_\theta\in \mathcal{F_A}$ that can realize the supremum: $f_\theta(G_j) = sign( \sum_{i \in I_j} \sigma_i)$ for some $\theta\in\Theta$. Then,
\begin{align}
    \frac{1}{m} \mathbb{E}_{\bm{\sigma}} \left[ \sup_{\Theta} \sum_{j=1}^p f_j(\Theta) \sum_{i \in I_j} \sigma_i \right]&=\frac{1}{m} \mathbb{E}_{\bm{\sigma}} \left[  \sum_{j=1}^p \vert\sum_{i \in I_j} \sigma_i \vert\right]
\end{align}
By the linearity of expectation, we have:
\begin{align}\label{eq:lin}
    \frac{1}{m} \mathbb{E}_{\bm{\sigma}} \left[  \sum_{j=1}^p \vert\sum_{i \in I_j} \sigma_i \vert\right]&=\frac{1}{m}   \sum_{j=1}^p \mathbb{E}_{\bm{\sigma}}\left[ \vert\sum_{i \in I_j} \sigma_i \vert\right]
\end{align}
By Khintchine's inequality we get the following bound on the expected value \cite{haagerup1981best}:
\begin{align}
    \mathbb{E}_{\bm{\sigma}}\left[ \vert\sum_{i \in I_j} \sigma_i \vert\right]\geq \sqrt{\frac{|I_j|}{2}}.
\end{align}
Then, substituting back into \cref{eq:lin}:
\begin{align}
    \frac{1}{m}   \sum_{j=1}^p \mathbb{E}_{\bm{\sigma}}\left[ \vert\sum_{i \in I_j} \sigma_i \vert\right]\geq \frac{1}{m}   \sum_{j=1}^p \sqrt{\frac{|I_j|}{2}}.
\end{align}
Given that $\mu_j:= \frac{|I_j|}{m}$, we can write the final lower bound for $\mathcal{R}_{S}(\mathcal{F})$ as:
\begin{align*}
    \mathcal{R}_{S}(\mathcal{F}) \geq \frac{1}{m} \sum_{j=1}^p \sqrt{\frac{|I_j|}{2}}= \sum_{j=1}^p \sqrt{\frac{\mu_j}{2m}}.
\end{align*}
\end{proof}

\dudleycoveringA*\label{app:proof-dudley}

\begin{proof}
The proof uses the Dudley entropy integral, adapting the bound to the $p$-dimensional structure of the function class $\mathcal{F}$.

Let us define a sequence of scales $\epsilon_k = \sqrt{m} \cdot 2^{-(k-1)}$ for $k \ge 1$. For each $k$, let $V_k$ be a minimal $\epsilon_k$-cover of the set of output vectors $\mathcal{F}_{|S}$ with respect to the $\ell_2$-norm, so $|V_k| = \mathcal{N}(\mathcal{F}_{|S}, \epsilon_k, ||\cdot||_2)$. Note that $\mathcal{F}_{|S}$ lies within the unit hypercube $[-1,1]^m$ therefore the cover is finite for $\epsilon_k > 0$.

For any function $f \in \mathcal{F}$, let $\bm{f}_{|S}$ be its corresponding vector in $\mathbb{R}^m$. Let $\bm{v}^k[\bm{f}]$ be a vector in $V_k$ such that $||\bm{f}_{|S} - \bm{v}^k[\bm{f}]||_2 \le \epsilon_k$. We decompose the vector $\bm{f}_{|S}$ using a telescoping sum:
$$ \bm{f}_{|S} = (\bm{f}_{|S} - \bm{v}^N[\bm{f}]) + \sum_{k=1}^{N-1} (\bm{v}^k[\bm{f}] - \bm{v}^{k+1}[\bm{f}]) + \bm{v}^1[\bm{f}] $$
The quantity to bound is $m \mathcal{R}_S(\mathcal{F}) = \mathbb{E}_{\bm{\sigma}} \sup_{f \in \mathcal{F}} \langle \bm\sigma, \bm{f}_{|S} \rangle$. Substituting the decomposition and using the triangle inequality for suprema gives:
\begin{align*}
    m \mathcal{R}_S(\mathcal{F}) \le \mathbb{E}_{\bm{\sigma}} \left[ \sup_{f \in \mathcal{F}} \langle \bm\sigma, \bm{f}_{|S} - \bm{v}^N[\bm{f}] \rangle \right] &+ \sum_{k=1}^{N-1} \mathbb{E}_{\bm{\sigma}} \left[ \sup_{f \in \mathcal{F}} \langle \bm\sigma, \bm{v}^k[\bm{f}] - \bm{v}^{k+1}[\bm{f}] \rangle \right] \\
    &+ \mathbb{E}_{\bm{\sigma}} \left[ \sup_{f \in \mathcal{F}} \langle \bm\sigma, \bm{v}^1[\bm{f}] \rangle \right]
\end{align*}
For the last term, note the scale $\epsilon_1 = \sqrt{m}$. Since $f(G_i) \in [-1,1]$, we have $||\bm{f}_{|S}||_2^2 = \sum_{i=1}^m f(G_i)^2 \le m$, which means $||\bm{f}_{|S}||_2 \le \sqrt{m}$. The zero vector $\bm{0}$ is in $\mathcal{F}_{|S}$ because $\bm{0} \in \mathcal{F}$, and for any $\bm{f}_{|S}$, $||\bm{f}_{|S} - \bm{0}||_2 \le \sqrt{m} = \epsilon_1$. Therefore, $V_1 = \{\bm{0}\}$ is a valid $\epsilon_1$-cover. We can choose $\bm{v}^1[\bm{f}] = \bm{0}$ for all $f \in \mathcal{F}$, causing the last term to be zero.

The first term is the expected supremum over the set of residual vectors, $\mathcal{G}_N = \{\bm{f}_{|S} - \bm{v}^N[\bm{f}] : f \in \mathcal{F}\}$, which is precisely $m \cdot \mathcal{R}_S(\mathcal{G}_N)$.

To bound the Rademacher complexity, $\mathcal{R}_S(\mathcal{G}_N)$, we can leverage two key properties of this class of residual functions. First, $\mathcal{G}_N$ inherits the partition structure from $\mathcal{F}$, which means $(f(G_i) - v_i^N[f]) = (f(G_j) - v_j^N[f])$ whenever $f(G_i) = f(G_j)$.

Given these properties, we can apply \Cref{app:real-valued_functions_bound}, \cref{eq:real_norm2}:

$$ \mathcal{R}_S(\mathcal{G}_N) \le \frac{\sup_{\bm{g} \in \mathcal{G}_N} ||\bm{g}||_2 \cdot \sqrt{p}}{m} $$
By definition of the cover $V_N$, we have $\sup_{g \in \mathcal{G}_N} ||\bm{g}||_2 \le \epsilon_N$. Substituting this in:
$$ \mathcal{R}_S(\mathcal{G}_N) \le \frac{\epsilon_N \sqrt{p}}{m} $$
Therefore, the term we need to bound is $m \cdot \mathcal{R}_S(\mathcal{G}_N)$, which gives:
$$ \mathbb{E}_{\bm{\sigma}} \left[ \sup_{f \in \mathcal{F}} \langle \bm\sigma, \bm{f}_{|S} - \bm{v}^N[\bm{f}] \rangle \right] = m \cdot \mathcal{R}_S(\mathcal{G}_N) \le \epsilon_N \sqrt{p} $$

For each $k \in \{1, \dots, N-1\}$, we bound $\mathbb{E}_{\bm{\sigma}} \left[ \sup_{f \in \mathcal{F}} \langle \bm\sigma, \bm{v}^k[\bm{f}] - \bm{v}^{k+1}[\bm{f}] \rangle \right]$.
Let $W_k = \{ \bm{v}^k[\bm{f}] - \bm{v}^{k+1}[\bm{f}] : f \in \mathcal{F} \}$. The expression is the Rademacher complexity of this finite set, $\mathbb{E}_{\bm{\sigma}} \left[ \sup_{\bm{w} \in W_k} \langle \bm\sigma, \bm{w} \rangle \right]$.
The size of this set is $|W_k| \le |V_k| \cdot |V_{k+1}|$. Since the covering numbers are monotonic ($\epsilon_k > \epsilon_{k+1} \implies |V_k| \le |V_{k+1}|$), we have $|W_k| \le |V_{k+1}|^2$.
The norm of any element $w \in W_k$ is bounded by the triangle inequality:
$$ ||\bm{w}||_2 = ||\bm{v}^k[\bm{f}] - \bm{v}^{k+1}[\bm{f}]||_2 \le ||\bm{v}^k[\bm{f}] - \bm{f}_{|S}||_2 + ||\bm{f}_{|S} - \bm{v}^{k+1}[\bm{f}]||_2 \le \epsilon_k + \epsilon_{k+1} $$
Since $\epsilon_k = 2\epsilon_{k+1}$, the norm is bounded by $3\epsilon_{k+1}$.
Applying Massart's Lemma:
\begin{align*}
    \mathbb{E}_{\bm{\sigma}} \left[ \sup_{w \in W_k} \langle \sigma, w \rangle \right] &\le (3\epsilon_{k+1}) \cdot \sqrt{2\log|W_k|} \le 3\epsilon_{k+1} \sqrt{2\log(|V_{k+1}|^2)} \\
    &= 3\epsilon_{k+1} \sqrt{4\log|V_{k+1}|} = 6\epsilon_{k+1} \sqrt{\log \mathcal{N}(\mathcal{F}_{|S}, \epsilon_{k+1}, ||\cdot||_2)}
\end{align*}

Combining the bounds on the residual and chain links yields:
\begin{align*}
m\mathcal{R}_S(\mathcal{F}) &\le \epsilon_N \sqrt{p} + \sum_{k=1}^{N-1} 6 \epsilon_{k+1} \sqrt{\log \mathcal{N}(\mathcal{F}_{|S}, \epsilon_{k+1}, ||\cdot||_2)}\\ &\le \epsilon_N \sqrt{p} + 12\sum_{k=1}^{N-1} (\epsilon_{k} -\epsilon_{k+1}) \sqrt{\log \mathcal{N}(\mathcal{F}_{|S}, \epsilon_k, ||\cdot||_2)}
\end{align*}
Using the standard step of bounding the sum with an integral and the monotonicity of the covering number:
$$ m\mathcal{R}_S(\mathcal{F}) \le \epsilon_N \sqrt{p} + 12 \int_{\epsilon_{N+1}}^{\sqrt{m}} \sqrt{\log\mathcal{N}(\mathcal{F}_{|S}, \epsilon, ||\cdot||_2)}d\epsilon $$

For any given $\alpha > 0$, we choose $N$ to be the largest integer such that $\epsilon_{N+1} > \alpha$. This implies $\epsilon_{N+2} \le \alpha$.
From the definition of the scales, we have $\epsilon_N = 2\epsilon_{N+1} = 4\epsilon_{N+2}$. This gives the bound $\epsilon_N \le 4\alpha$.
The lower limit of the integral, $\epsilon_N$, is greater than $\alpha$. Therefore, the integral is over a smaller domain than $[\alpha, \sqrt{m}]$, so we can bound it:
$$ \epsilon_N \sqrt{p} + 12\int_{\epsilon_N}^{\sqrt{m}} \sqrt{\log\mathcal{N}(\dots)}d\epsilon \le 4\alpha \sqrt{p} +  12\int_{\alpha}^{\sqrt{m}} \sqrt{\log\mathcal{N}(\dots)}d\epsilon $$
Substituting these into the main inequality:
$$ m\mathcal{R}_S(\mathcal{F}) \le 4\alpha\sqrt{p} + 12 \int_{\alpha}^{\sqrt{m}} \sqrt{\log\mathcal{N}(\mathcal{F}_{|S}, \epsilon, ||\cdot||_2)}d\epsilon $$
Dividing by $m$ gives the bound for our chosen $\alpha$. 
This improves upon the classical bound \cite{bartlett2017spectrally} replacing $4\alpha/\sqrt{m}$ with $4\alpha\sqrt{p}/m$. 
As this holds for any $\alpha > 0$, we may take the infimum to find the tightest bound:
$$ \mathcal{R}_S(\mathcal{F}) \le \inf_{\alpha > 0} \left( \frac{4\alpha\sqrt{p}}{m} + \frac{12}{m}\int_{\alpha}^{\sqrt{m}} \sqrt{\log\mathcal{N}(\mathcal{F}_{|S}, \epsilon, ||\cdot||_2)}d\epsilon \right) $$
This completes the proof.
\end{proof}

\upradmetricB*

\begin{proof}
    Let $({S}, \mathrm{d}_{S})$ be the pseudo-metric subspace with set S and metric $\mathrm{d}_{S}: S\times S \to \mathcal{H}$ where $\mathcal{H}$ is a normed space. Consider the $\varepsilon$-ball $B_{\varepsilon}(\hat{G}_j)$ with centroid $\hat{G}_j$.
    Let $I_j$ be the set of indices of graphs in $B_{\varepsilon}(\hat{G}_j)$: 
    $$
        I_j := \{ i \in [m] : G_i \in B_{\varepsilon}(\hat{G}_j) \}
    $$
    Additionally, let $\lambda$ be the \emph{Lipschitz certificate} of $\mathcal{F_A}$  such that
    $$
        \exists \lambda > 0 : \forall G,H \in \mathcal{G_{S}}, \; | f_\theta(G) - f_\theta(H) | \leq \lambda \cdot d_{S}(G, H)
    $$
    
    The empirical Rademacher can then be re-arranged in terms of the distance of each individual graph to its centroid:
    \begin{align}
        \mathcal{R}_{S}(\mathcal{F_A}) &= \mathbb{E}_{\bm{\sigma}} \left[ \sup_{\Theta}\frac{1}{m} \sum_{i=1}^m \sigma_i f_\theta(G_i) \right] \nonumber\\
        &= \frac{1}{m} \mathbb{E}_{\bm{\sigma}} \left[ \sup_{\Theta} \sum_{j=1}^{p_{\varepsilon}} \sum_{i \in I_j} \sigma_i \left( f_\theta(G_i) - f_\theta(\hat{G}_j) + f_\theta(\hat{G}_j) \right) \right] \nonumber\\
        &= \frac{1}{m} \mathbb{E}_{\bm{\sigma}} \left[ \sup_{\Theta}\left[ \sum_{j=1}^{p_{\varepsilon}} \sum_{i \in I_j} \sigma_i \left( f_\theta(G_i) - f_\theta(\hat{G}_j) \right) + \sum_{j=1}^{p_{\varepsilon}} \sum_{i \in I_j} \sigma_i f_\theta(\hat{G}_j)\right] \right] \nonumber\\
        &= \frac{1}{m} \mathbb{E}_{\bm{\sigma}} \left[ \sup_{\Theta} \left[\sum_{j=1}^{p_{\varepsilon}} \sum_{i \in I_j} \sigma_i \left( f_\theta(G_i) - f_\theta(\hat{G}_j) \right) + \sum_{j=1}^{p_{\varepsilon}} f_\theta(\hat{G}_j) \sum_{i \in I_j} \sigma_i \right] \right] \nonumber\\ \label{eq:radmetric_setup}
    \end{align}
    
    \paragraph*{Upper Bound.}
    First, let $\delta_i = f_\theta(G_i) - f_\theta(\hat{G}_j)$. Then, by the triangle inequality 
    \begin{align}
    \label{eq:first_proof_bm}
        \sum_{i \in I_j} \sigma_i \delta_i \leq \left|\sum_{i \in I_j} \sigma_i \delta_i \right| \leq \sum_{i \in I_j} |\sigma_i \delta_i| \leq \sum_{i \in I_j} |\sigma_i| |\delta_i|
    \end{align}
    
    We know that for each centroid $\hat{G}_j$, the distance to each $G_i \in B_{\varepsilon}(\hat{G}_j)$ is at most $\varepsilon$. By Lipschitz we also know that $\delta_i \leq \lambda d_{S}(G_i, \hat{G}_j) \leq \lambda \varepsilon$. Applying the triangle inequality of the supremum, linearity of expectation, and the previous observations, we can divide the two sums to obtain 
    \begin{align}
    \label{eq:before_up_bound}
    \mathcal{R}_{S}(\mathcal{F_A}) &\leq \frac{1}{m} \mathbb{E}_{\bm{\sigma}} \left[ \sup_{\Theta} \sum_{j=1}^{p_{\varepsilon}} \sum_{i \in I_j} |\sigma_i| \varepsilon \cdot \lambda + \sup_{\Theta} \sum_{j=1}^{p_{\varepsilon}} \sum_{i \in I_j} \sigma_i f_\theta(\hat{G}_j)   \right] \nonumber\\
    &\leq \frac{1}{m} \mathbb{E}_{\bm{\sigma}} \left[ \sup_{\Theta} \sum_{j=1}^{p_{\varepsilon}} \sum_{i \in I_j} |\sigma_i| \varepsilon \cdot \lambda \right] + \frac{1}{m}\mathbb{E}_{\bm{\sigma}} \left[ \sup_{\bm{\Theta}} \sum_{j=1}^{p_{\varepsilon}} \sum_{i \in I_j} \sigma_i f_\theta(\hat{G}_j)   \right] \nonumber\\
    &\leq \frac{\varepsilon \cdot \lambda}{m} \mathbb{E}_{\bm{\sigma}} \left[  \sum_{j=1}^{p_{\varepsilon}} \sum_{i \in I_j} |\sigma_i|  \right] + \frac{1}{m}\mathbb{E}_{\bm{\sigma}} \left[ \sup_{\Theta} \sum_{j=1}^{p_{\varepsilon}} \sum_{i \in I_j} \sigma_i f_\theta(\hat{G}_j)   \right] \nonumber\\
    &\leq \varepsilon \cdot \lambda + \frac{1}{m} \mathbb{E}_{\bm{\sigma}} \left[ \sup_{\Theta} \sum_{j=1}^{p_{\varepsilon}} \sum_{i \in I_j} \sigma_i f_\theta(\hat{G}_j)   \right] \nonumber\\
    &\leq \varepsilon \cdot \lambda + \frac{1}{m} \mathbb{E}_{\bm{\sigma}} \left[ \sup_{\Theta} \sum_{j=1}^{p_{\varepsilon}}  f_\theta(\hat{G}_j) \sum_{i \in I_j} \sigma_i   \right] \nonumber\\
    \end{align}
    
    Let $Z_j = \sum_{i \in I_j} \sigma_i$.
Given that the output of the $f_\theta$ is in $[-1,1]$, then 
${f}_\theta(\hat G_j)Z_j\le |Z_j|$ for all ${f}_\theta\in {\mathcal{F_A}}:$
    \begin{equation*}
     \sup_{\Theta} \sum_{j=1}^{p_{\varepsilon}}  f_\theta(\hat{G}_j) Z_j\le\sum_{j=1}^{p_{\varepsilon}} \left| Z_j\right|   
    \end{equation*}
    
    If we substitute this back into \cref{eq:first_proof_bm}, we obtain the final bound:
    \begin{align*}
        \mathcal{R}_{S} &\leq \varepsilon \cdot \lambda + \frac{1}{m} \mathbb{E}_{\bm{\sigma}} \left[  \sum_{j=1}^{p_{\varepsilon}}  \left| Z_j   \right| \right] \\
        &\leq \varepsilon \cdot \lambda + \frac{1}{m}  \sum_{j=1}^{p_{\varepsilon}} \mathbb{E}_{\bm{\sigma}} \left[   \left| Z_j   \right| \right] \\
        &\leq \varepsilon \cdot \lambda + \frac{1}{m}  \sum_{j=1}^{p_{\varepsilon}} \sqrt{\sum_{i \in I_j} 1} \\
        &= \varepsilon \cdot \lambda + \frac{1}{m}  \sum_{j=1}^{p_{\varepsilon}} \sqrt{|I_j|} \\
        &= \varepsilon \cdot \lambda +   \sum_{j=1}^{p_{\varepsilon}} \sqrt{\frac{\mu_j}{{m}}} \\
    \end{align*}
    
    \paragraph*{Lower Bound.}
    Starting from \cref{eq:radmetric_setup}, let $T_1(\Theta) = \sum_{j=1}^{p_{\varepsilon}} \sum_{i \in I_j} \sigma_i \delta_i$ and $T_2(\Theta) = \sum_{j=1}^{p_{\varepsilon}} f_\theta(\hat{G}_j) \sum_{i \in I_j} \sigma_i$, where $\delta_i = f_\theta(G_i) - f_\theta(\hat{G}_j)$. By the subadditivity of the supremum, we have $\sup_{\Theta} (T_1(\Theta) + T_2(\Theta)) \geq \sup_{\Theta} T_2(\Theta) - \sup_{\Theta} (-T_1(\Theta))$. Furthermore, for any real value, $-x \leq |x|$, meaning we can bound the subtracted term as $\sup_{\Theta} (-T_1(\Theta)) \leq \sup_{\Theta} |T_1(\Theta)|$.
    \begin{align}
        \mathcal{R}_{S}(\mathcal{F_A}) &= \frac{1}{m} \mathbb{E}_{\bm{\sigma}} \left[ \sup_{\Theta} \left( T_1(\Theta) + T_2(\Theta) \right) \right] \nonumber\\
        &\geq \frac{1}{m} \mathbb{E}_{\bm{\sigma}} \left[ \sup_{\Theta} T_2(\Theta) \right] - \frac{1}{m} \mathbb{E}_{\bm{\sigma}} \left[ \sup_{\Theta} |T_1(\Theta)| \right] \nonumber\\
        &\geq - \frac{1}{m} \mathbb{E}_{\bm{\sigma}} \left[ \sup_{\Theta} |T_1(\Theta)| \right] + \frac{1}{m} \mathbb{E}_{\bm{\sigma}} \left[ \sup_{\Theta} T_2(\Theta) \right]  \nonumber\\
    \end{align}
    
    From the Lipschitzianity of $f_\theta$ and the $\varepsilon$-ball covering of each centroid, we know that $|\delta_i| \leq \lambda d_{S}(G_i, \hat{G}_j) \leq \lambda \varepsilon$. Substituting this absolute value upper bound back into the inequality with $T_1(\Theta)$, we obtain:
    \begin{align}
        \mathcal{R}_{S}(\mathcal{F_A}) &\geq - \frac{1}{m} \mathbb{E}_{\bm{\sigma}} \left[ \sup_{\Theta} |T_1(\Theta)| \right] + \frac{1}{m} \mathbb{E}_{\bm{\sigma}} \left[ \sup_{\Theta} T_2(\Theta) \right]  \nonumber\\
        &\geq - \frac{1}{m} \mathbb{E}_{\bm{\sigma}} \left[ \sup_{\Theta} \left| \sum_{j=1}^{p_{\varepsilon}} \sum_{i \in I_j} \sigma_i \delta_i \right| \right] + \frac{1}{m} \mathbb{E}_{\bm{\sigma}} \left[ \sup_{\Theta} T_2(\Theta) \right]  \nonumber\\
        &\geq - \frac{1}{m} \mathbb{E}_{\bm{\sigma}} \left[ \sup_{\Theta} \sum_{j=1}^{p_{\varepsilon}} \sum_{i \in I_j} \lambda \varepsilon \right] + \frac{1}{m} \mathbb{E}_{\bm{\sigma}} \left[ \sup_{\Theta} T_2(\Theta) \right]  \nonumber\\
        &\geq  - \lambda \varepsilon + \frac{1}{m} \mathbb{E}_{\bm{\sigma}} \left[ \sup_{\Theta} T_2(\Theta) \right] \nonumber\\
    \end{align}
    
    If we similarly let $Z_j=\sum_{i \in I_j} \sigma_i$, then $T_2(\Theta)$ is the inner sum $\sum_{j=1}^{p_{\varepsilon}} f_\theta(\hat{G}_j) Z_j$. Under the assumption that $\mathrm{d}$ is a $T-$equivalent metric (see \Cref{def:T-equivalent}), the function class is able to assign an arbitrary output $f_\theta(\hat{G}_j)$ for each centroid. {Now, we assume by hypothesis that $\mathcal{F_A}$ contains maximally expressive functions, that is functions that match the expressive power of the bounding invariant.}
    Thus, the function that maximizes the sum is the one whose output perfectly aligns with the sign of the Rademacher noise sum, i.e. $f_\theta(\hat{G}_j) = \text{sign}(Z_j)$. Using this and the linearity of expectation, we obtain:
    \begin{align}
        \frac{1}{m} \mathbb{E}_{\bm{\sigma}} \left[ \sup_{\Theta} T_2(\Theta) \right] &= \frac{1}{m} \mathbb{E}_{\bm{\sigma}} \left[ \sum_{j=1}^{p_{\varepsilon}} |Z_j| \right] \nonumber\\
        &= \frac{1}{m} \mathbb{E}_{\bm{\sigma}} \left[ \sum_{j=1}^{p_{\varepsilon}} |\sum_{i \in I_j} \sigma_i| \right]\nonumber\\
        &= \frac{1}{m} \sum_{j=1}^{p_{\varepsilon}} \mathbb{E}_{\bm{\sigma}} \left[ |\sum_{i \in I_j} \sigma_i| \right] \nonumber\\
    \end{align}
    Using Khintchine's inequality, we can bound the expected value \cite{haagerup1981best}:
    \begin{align}
        \mathbb{E}_{\bm{\sigma}} \left[ |\sum_{i \in I_j} \sigma_i| \right] &\geq \frac{1}{\sqrt{2}}\sqrt{\sum_{i \in I_j} 1^2} \nonumber\\
        &\geq \sqrt{\frac{|I_j|}{2}} \nonumber\\
    \end{align}
    Substituting this back into the main expression yields the final lower bound:
    \begin{align}
        \mathcal{R}_{S}(\mathcal{F_A}) &\geq - \lambda \varepsilon + \frac{1}{m} \mathbb{E}_{\bm{\sigma}} \left[ \sup_{\Theta} T_2(\Theta) \right] \nonumber\\
        &= - \lambda \varepsilon + \frac{1}{m} \sum_{j=1}^{p_{\varepsilon}} \mathbb{E}_{\bm{\sigma}} \left[ |\sum_{i \in I_j} \sigma_i| \right] \nonumber\\
        &\geq - \lambda \varepsilon + \frac{1}{m} \sum_{j=1}^{p_{\varepsilon}} \sqrt{\frac{|I_j|}{2}} \nonumber\\
        &= - \lambda \varepsilon + \sum_{j=1}^{p_{\varepsilon}} \sqrt{\frac{\mu_j^{\varepsilon}}{2m}} \nonumber\\
    \end{align}
\end{proof}

\subsection{Stability}\label{app:stability-proofs}

\begin{lemma}
\label{lemma:sup-difference}
Let $X$ be a nonempty set and let $f,g:X\to\mathbb{R}$ be two real–valued functions.  Then
\begin{equation}
\bigl|\sup_{x\in X}f(x)\;-\;\sup_{x\in X}g(x)\bigr|
\le\sup_{x\in X}\bigl|f(x)-g(x)\bigr|.
\end{equation}
\end{lemma}

We are now ready to prove \cref{thm: stability}.

\stability*\label{app:proof-stability}
\begin{proof}
We begin by proving the stability of the Rademacher complexity in the case of a discrete pseudo-metric, such as $\mathrm{1-WL}$.
\paragraph{Proof in the discrete setting.}
Let $T : \mathcal{G} \to \mathcal
C$ be a graph invariant that induces a partition of $S$ and $S'$ into equivalence classes. Let $c_j$ be the unique value assigned by $T$ to all graphs belonging to class $C_j$. Let $\mathcal{C}(S)$ be the set of labels assigned by $T$ to the graphs in $S$, and $\mathcal{C}(S')$ is similarly defined. We set 
\begin{equation}\label{eq:partition-multiplicity}
I_j(S) := \{\, i \in [m] : G_i\in S \text{ and } T(G_i) = c_j \,\},
\end{equation}
and similarly for $S'$.
Since the architecture $\mathcal{A}$ is as good as $T$ in distinguishing graphs, the output of $f_\theta \in \mathcal{F_A}$ is assumed to be identical for all graphs in the same partition, i.e. $f_\theta(G_i)=f_\theta(G_k)$ for any $i,k \in I_j$. 
The empirical Rademacher complexity can be expressed in terms of equivalence classes induced by the graph invariant as:
\begin{equation}\label{eq:a}
\mathcal{R}_{S}(\mathcal{F_A}) := \mathbb{E}_{\bm{\sigma}}\left[\sup_{\Theta}\frac{1}{m}\sum_{i=1}^{m}\sigma_i f_\theta(G_i)\right]= \mathbb{E}_{\bm{\sigma}}\left[\sup_{\Theta}\frac{1}{m}\sum_{c_j\in\mathcal{C}(S)}f_j(\Theta)\sum_{i\in I_j(S)} \sigma_i\right],
\end{equation}

where$f_j(\Theta) $ is the constant output for any graph in partition $C_j$. An analogous expression holds for $S'$. 

To compare the empirical Rademacher Complexity of two different samples $S$ and $S'$, we require the definition to be invariant under permutations of the graph indices (otherwise, the Rademacher complexity of the same sample could change simply by reordering its elements). To ensure this invariance, we rewrite $\sum_{i\in I_j(S)} \sigma_i$ in terms of labels $c_j$ and their multiplicity, that is $n_j(S)=|I_j(S)|$:
$$\sum_{i\in I_j(S)} \sigma_i= \sum_{i=1}^{n_j(S)} \sigma_{j,i}$$where for every color $c_j$, the sequence $(\sigma_{j,i})_{i\geq 1}$ is shared across samples (with each sample using only the first $n_{j}(S)$ terms, depending on its multiplicity).
Additionally, we re-index \cref{eq:a} over the whole set of invariant labels $\mathcal{C}$.
\begin{equation}
\left\vert\mathcal{R}_{S}(\mathcal{F_A})-\mathcal{R}_{S'}(\mathcal{F_A}) \right\vert \le \mathbb{E}_{\bm\sigma}
\left\vert\sup_{\Theta}\frac{1}{m}\sum_{c_j\in\mathcal{C}}f_j(\Theta)\sum_{i=1}^{n_j(S)} \sigma_{j,i} - \sup_{\Theta}\frac{1}{m}\sum_{c_j\in\mathcal{C}}f_j(\Theta)\sum_{i=1}^{n_j(S')} \sigma_{j,i}\right\vert.
\end{equation}
Note that if a specific equivalence class $c$ occurs only in one sample, its multiplicity is zero in the other, so the corresponding contribution vanishes.
Now by \cref{lemma:sup-difference}, we can upper-bound the difference of suprema by the supremum of the differences, yielding: 
\begin{equation}\label{eq:diff}
\begin{aligned}
\left\vert\mathcal{R}_{S}(\mathcal{F_A})-\mathcal{R}_{S'}(\mathcal{F_A}) \right\vert &\leq \mathbb{E}_{\bm{\sigma}} \left[\sup_{\Theta}\left\vert\frac{1}{m}\sum_{c_j\in\mathcal{C}}f_j(\Theta)\sum_{i=1}^{n_j(S)} \sigma_{j,i} - \frac{1}{m}\sum_{c_j\in\mathcal{C}}f_j(\Theta)\sum_{i=1}^{n_j(S')} \sigma_{j,i} \right\vert\right] \\ 
&\leq \mathbb{E}_{\bm{\sigma}} \left[\sup_{\Theta}\frac{1}{m}\sum_{c_j\in\mathcal{C}}\left\vert f_j(\Theta)\right\vert\left\vert \sum_{i=1}^{n_j(S)} \sigma_{j,i} -  \sum_{i=1}^{n_c(S')} \sigma_{j,i}  \right\vert \right].  
\end{aligned}
\end{equation}
Now, set $\min{n_j} := \min(n_j(S),n_j(S'))$ and $\max\,{n_j} := \max(n_j(S),n_j(S'))$ and, given that the Rademacher sequence $(\sigma_{j,i})_{i\geq 1}$ is shared across samples, we can write: $$\left\vert \sum_{i=1}^{n_j(S)} \sigma_{j,i} -  \sum_{i=1}^{n_j(S')} \sigma_{j,i} \right\vert= \left\vert \sum_{i=min\,{n_j}+1}^{max\,{n_j}}\sigma_i\right\vert \leq \left\vert n_j(S) - n_j(S')\right\vert$$
Hence, the bound in Eq.\ref{eq:diff} can be rewritten in terms of color multiplicities, independently of the $\sigma_i$'s: 

$$\left\vert\mathcal{R}_{S}(\mathcal{F_A})-\mathcal{R}_{S'}(\mathcal{F_A}) \right\vert 
\leq  \sup_{f\in \mathcal{F_A}}\frac{1}{m}\sum_{c_j\in\mathcal{C}}\left\vert f_\theta(G_j)\right\vert\left\vert n_j - n'_j\right\vert.$$

The final bound is obtained using the fact that $\sup\limits_{f\in \mathcal{F_A}}\vert f_\theta(G)\vert \leq 1$:
$$ \left\vert\mathcal{R}_{S}(\mathcal{F_A})-\mathcal{R}_{S'}(\mathcal{F_A}) \right\vert
\leq \frac{1}{m}\sum_{c_j\in\mathcal{C}}\left\vert n_j - n'_j\right\vert .$$

\paragraph{Proof in the pseudo-metric space}
Now we generalize the proof to the case where $S$ is endowed with a continuous pseudo-metric $d$, and the function class is $\lambda-$Lipschitz with respect to this metric. Note that the results established above did not require the Lipschitzianity assumption, since the $T-$metric is inherently tied to the architecture. \\

Let $S=\{G_1, \cdots, G_m\}$, $S'=\{G'_1, \cdots, G'_m\}$ be two datasets drawn from the same distribution $\mathcal{D}^m$. Let $\pi:[m]\to [m]$ be an arbitrary permutation of the indices. We denote by $\pi(S)$ the permuted sample
$\pi(S) := \{G_{\pi(1)},\ldots, G_{\pi(m)}\}$.
The core intuition of the proof relies on the permutation invariance of the Rademacher complexity, that is $\mathcal{R}_S=\mathcal{R}_{\pi(S)}$, because the expectation is taken over independent and identically distributed (i.i.d.) random variables $\sigma$ the order of the graphs in the dataset does not affect the expected supremum.
\begin{align}
    \left\vert\mathcal{R}_{S}(\mathcal{F_A})-\mathcal{R}_{S'}(\mathcal{F_A}) \right\vert=\left\vert\mathcal{R}_{S}(\mathcal{F_A})-\mathcal{R}_{\pi(S')}(\mathcal{F_A}) \right\vert = \\\left\vert\mathbb{E}_{\bm{\sigma}} \left[ \sup_{\Theta}\frac{1}{m} \sum_{i=1}^m \sigma_i f_\theta(G_i) \right] - \mathbb{E}_{\bm{\sigma}} \left[ \sup_{\Theta}\frac{1}{m} \sum_{i=1}^m \sigma_{\pi(i)} f_\theta(G'_{\pi(i)}) \right] \right\vert&\nonumber.
\end{align}
Given that the $\sigma'$s don't depend on the order of the graphs, we can align $G_i$ and $G'_{\pi(i)}$ under the exact same Rademacher variable $\sigma_i$, which can then be factored out:
\begin{equation*}
    \left\vert\mathcal{R}_{S}(\mathcal{F_A})-\mathcal{R}_{S'}(\mathcal{F_A}) \right\vert 
\le \frac{1}{m} \left\vert\mathbb{E}_{\bm{\sigma}} \left[ \sup_{\Theta}\left(\sum_{i=1}^m \sigma_i \left(f_\theta(G_i)  -  f_\theta(G'_{\pi(i)})\right)\right) \right] \right\vert
\end{equation*}
The $\le$ comes from the subadditivity of the supremum.
By triangular inequality we get:
\begin{align*}
    \left\vert\mathcal{R}_{S}(\mathcal{F_A})-\mathcal{R}_{S'}(\mathcal{F_A}) \right\vert 
\le \frac{1}{m} \mathbb{E}_{\bm{\sigma}} \left[ \sup_{\Theta}\left\vert\sum_{i=1}^m \sigma_i \left(f_\theta(G_i)  -  f_\theta(G'_{\pi(i)})\right)\right\vert\right] \\
\le \frac{1}{m} \mathbb{E}_{\bm{\sigma}} \left[ \sup_{\Theta}\left(\sum_{i=1}^m |\sigma_i| |f_\theta(G_i)  -  f_\theta(G'_{\pi(i)})|\right) \right] 
&
\end{align*}
Now, given that the $|\sigma_i|=1$ for all $i$, and by $\lambda-$Lipschitzianity of $f$, we get:
\begin{equation*}
\left\vert\mathcal{R}_{S}(\mathcal{F_A})-\mathcal{R}_{S'}(\mathcal{F_A}) \right\vert 
\le \frac{1}{m}\sum_{i=1}^m \lambda d(G_i,G_{\pi(i)}).
\end{equation*}
Given that this bound holds for any permutation $\pi$, it's valid also for the one that minimizes the distances between the graphs in $S$ and $S'$. This corresponds to the solution to the optimal transport problem.
\begin{equation*}
\left\vert\mathcal{R}_{S}(\mathcal{F_A})-\mathcal{R}_{S'}(\mathcal{F_A}) \right\vert 
\le \min_{\pi}{\frac{1}{m}\sum_{i=1}^m \lambda d(G_i,G_{\pi(i)})}.
\end{equation*}
This last term is exactly the $1$-Wasserstein distance between the uniform empirical measure supported on the two datasets.
\end{proof}
\stabilitygeneralization*\label{app:proof-cross-dataset}
\begin{proof} By \cite[Theorem 3.3]{mohri2008rademacher}, the generalization gap on the target sample $S'$ is bounded by:
\begin{equation*}
    e^{gen}(f, S') := L(f) - L_{S'}(f) \le 2\gamma \mathcal{R}_{S'}(\mathcal{F}_{\mathcal{A}}) + 3\sqrt{\frac{\ln(2/\delta)}{2m}}.
\end{equation*}
From \Cref{thm: stability}, we know that the difference in empirical Rademacher complexities is bounded by the Wasserstein distance over the quotient space: $|\mathcal{R}_{S'}(\mathcal{F}_{\mathcal{A}}) - \mathcal{R}_S(\mathcal{F}_{\mathcal{A}})| \le \alpha W_1(S, S')$. Considering only the upper bound of the absolute value, we have $\mathcal{R}_{S'}(\mathcal{F}_{\mathcal{A}}) \le \mathcal{R}_S(\mathcal{F}_{\mathcal{A}}) + \alpha W_1(S, S')$. Substituting this into the generalization bound yields the result. 
\end{proof}
 \section{Rademacher bounds on the quotient space}\label{app:quotient}A graph invariant $T$ naturally induces an equivalence relation $\equiv_T$ over the space of graphs $\mathcal{G}$, such that $G\equiv_TG' \iff T(G)=T(G')$. Consequently, this relation partitions any finite sample $S=\{G_1, \ldots,G_m\}$ into distinct equivalence classes, where each class $C_j=\{G\in S \mid T(G)=c_j\}$ corresponds to a unique invariant output $c_j$. The quotient space is defined as the set of these equivalent classes, denoted by  $S/_{\equiv_T}=\{C_1, \ldots, C_p\}$.
For readability, in the following we refer to the quotient space $S/_{\equiv_T}$ as $S/_{{\equiv}}$.
Let $\mathcal A$ be a GNN architecture whose expressive power is upper-bounded by this invariant (i.e., $\mathcal{A}\preceq T$). By definition, this expressivity constraint implies that any function $f\in \mathcal{F_A}$ is constant over graphs belonging to the same equivalence class. 

\paragraph{Functions and metrics induced in the quotient space.}

We begin by formally defining the action of our hypothesis class on the quotient space.
\begin{definition}[Induced quotient function]
Let $\equiv_T$ be the equivalence relation induced by $T$, and let $S/_{{\equiv}}= \{C_1, \ldots, C_p\}$ be the corresponding quotient space consisting of $p$ equivalence classes. For any $f_\theta \in \mathcal{F}_{\mathcal{A}}$, we define the induced function $\bar{f}_\theta: S/_{{\equiv}} \to [-1,1]$ as:
$$ \bar{f}_\theta(C_j) = f_\theta(G), \quad \text{for any } G \in C_j. $$
\end{definition}
The function $\bar{f}_\theta$ is well-defined because $f_\theta$ is constant over the equivalence class $C_j$, making the output independent of the chosen representative of the class. We call $\bar{\mathcal{F_A}}:=\{\bar{f}_\theta:  f_\theta \in \mathcal{F_A} \}$

Next, we assume the original space $\mathcal{G}$ is endowed with a pseudo-metric $d$ such that $\mathcal{F}_{\mathcal{A}}$ is Lipschitz with certificate $\lambda$. That is, for all $f_\theta\in \mathcal{F_A}$:
$$ |f_\theta(G) - f_\theta(G')| \le \lambda \cdot d(G, G') \quad \forall G, G' \in \mathcal{G}. $$
We can equip the quotient space $S/_{{\equiv}}$ with an induced pseudo-metric  
$\bar d$ under which the Lipschitz continuity is preserved. The following lemma guarantees that such valid pseudo-metrics exist.

\begin{lemma}[Lipschitz continuity on quotient metrics]\label{lemma:quotient_lipschitz}
Let $d$ be a pseudo-metric on $\mathcal{G}$ under which $f_\theta \in \mathcal{F}_{\mathcal{A}}$ is $\lambda$-Lipschitz. Let $\bar{d}: S/_{\equiv_T} \times S/_{\equiv_T} \to \mathbb{R}_{\ge 0}$ be a pseudo-metric defined on the quotient space such that for any two classes $C_i, C_j \in S/_{\equiv_T}$:
$$ \bar{d}(C_i, C_j) := \max\{\max_{G \in C_i}\min_{G' \in C_j} d(G, G'),\max_{G' \in C_j}\min_{G \in C_i} d(G, G')\}. $$
That is, the Hausdorff distance between the equivalence classes. Then, the induced function $\bar{f}_\theta$ is $\lambda$-Lipschitz on the quotient space with respect to $\bar{d}$:
$$ |\bar{f}_\theta(C_i) - \bar{f}_\theta(C_j)| \le \lambda \cdot \bar{d}(C_i, C_j). $$
\end{lemma}

\begin{proof}
By the definition of the induced function $\bar{f}_\theta$ and the Lipschitz continuity of $f_\theta$ on the original space, we have:
$$ |\bar{f}_\theta(C_i) - \bar{f}_\theta(C_j)| = |f_\theta(G) - f_\theta(G')| \le \lambda \cdot d(G, G') $$
for any arbitrary choice of representatives $G \in C_i$ and $G' \in C_j$. Since this inequality holds for \textit{all} pairs in the Cartesian product $C_i \times C_j$, it must strictly hold for the infimum of the distances:
$$ |\bar{f}_\theta(C_i) - \bar{f}_\theta(C_j)| \le \lambda \cdot \min_{G \in C_i, G' \in C_j} d(G, G')\le \lambda \cdot \bar{d}(C_i, C_j). $$
By definition, the minimum pairwise distance between two sets is a lower bound for their Hausdorff distance, yielding the result. 

\end{proof}

\begin{remark}\label{rem:metrics_quotient}
    Since we evaluate the empirical Rademacher complexity on a finite sample S, the equivalence classes are finite sets. Consequently, the infimum distance between any two classes is always attained as a minimum. Valid alternative candidates for the quotient pseudo-metric $\bar d$ that satisfy Lemma \ref{lemma:quotient_lipschitz} are the Hausdorff distance between sets, or the average and maximum pairwise distances, as all of these are inherently bounded below by the minimum pairwise distance.
\end{remark} 
\subsection{Bounding the Rademacher complexity}
Having established the metric structure, we can now bound the empirical Rademacher complexity $\mathcal{R}_S(\mathcal{F}_{\mathcal{A}})$ using the covering numbers of the quotient space $S/_{\equiv_T}$.

\begin{theorem}\label{app:bound_quotient}
\label{thm:quotient_bounds}
Let $S = \{G_1, \dots, G_m\}$ be a sample of $m$ graphs, and let $\mathcal{F}_{\mathcal{A}}$ be the hypothesis class of a GNN architecture bounded by a graph invariant $T$ (i.e., $\mathcal{A} \le T$), with outputs in $[-1, 1]$. 
Let $(S/_{\equiv},\bar{d})$ be the quotient space induced by the invariant $T$ and equipped with a pseudo-metric $\bar{d}$, under which the induced hypothesis class $\bar{\mathcal{F}}_{\mathcal{A}}$ is $\lambda$-Lipschitz. Let $\bar{p}_\varepsilon = \mathcal{N}(S/_{\equiv_T}, \varepsilon, \bar{d})$ denote the covering number of the quotient space at radius $\varepsilon > 0$, and let $\bar \mu_k^\varepsilon$ be the empirical frequency of graphs in $S$ mapped to the $k$-th partition of the $\varepsilon$-cover of $S/_\equiv$. 

Then, the empirical Rademacher complexity of $\mathcal{F}_{\mathcal{A}}$ is upper bounded by:
$$ \mathcal{R}_S(\mathcal{F}_{\mathcal{A}}) \le \sum_{k=1}^{\bar{p}_\varepsilon} \sqrt{\frac{\bar \mu_k^\varepsilon}{m}} + \lambda \varepsilon. $$
Furthermore, if $\mathcal{F}_{\mathcal{A}}$ contains maximally expressive functions matching the distinguishing power of $T$, the complexity is tightly bounded from both sides:
$$ \sum_{k=1}^{\bar{p}_\varepsilon} \sqrt{\frac{\bar \mu_k^\varepsilon}{2m}} - \lambda \varepsilon \le \mathcal{R}_S(\mathcal{F}_{\mathcal{A}}) \le \sum_{k=1}^{\bar{p}_\varepsilon} \sqrt{\frac{\bar \mu_k^\varepsilon}{m}} + \lambda \varepsilon. $$
\end{theorem}
    
\begin{proof} 
Let $\{C_1, \dots, C_p\}$ be the $p$ distinct equivalence classes present in the sample $S$. We can rewrite the Rademacher complexity by projecting each $G_i \in S$ onto its class $C_j \in S/_{\equiv_T}$ . Using the property that if $G_i \in C_j$ then $f_\theta(G_i) = \bar{f}_\theta(C_j)$ , we have:
    $$\mathcal R_S(\mathcal F_{\mathcal A})
=
\frac{1}{m}
\mathbb E_\sigma\sup_{\Theta}\sum_{i=1}^m \sigma_i f_\theta(G_i)
= \frac{1}{m}
\mathbb E_\sigma
\sup_{\Theta}\sum_{j=1}^p \sum_{i\in I_j}\sigma_i\bar{f}_\theta(C_j),$$
where $I_j=\{i \in[m]: G_i \in C_j \}$. \\
Let
$\{\hat C_1,\dots,\hat C_{\bar p_\varepsilon}\}$ be
a $\varepsilon$-cover of $S_{\diagup_{\equiv}}$ with respect to the induced metric $\bar d$, with
$\bar p_\varepsilon=N(S_{\diagup_{\equiv}},\varepsilon, \bar d)$.
For every class $C_j$ contained in the ball $B_\varepsilon(\hat{C}_k)$, we decompose the function value as $\bar{f}_\theta(C_j) = (\bar{f}_\theta(C_j) - \bar{f}_\theta(\hat{C}_k)) + \bar{f}_\theta(\hat{C}_k)$. Let $\bar{I}_k^\varepsilon = \{j \in [p] : C_j \in B_\varepsilon(\hat{C}_k)\}$ denote the set of classes assigned to centroid $\hat{C}_k$.
We can thus cluster the sum over the balls of the $\varepsilon$-cover:

$$\sum_{j=1}^p \sum_{i\in I_j}\sigma_i\bar{f}_\theta(C_j) = \sum_{k=1}^{\bar p_\varepsilon}\sum_{j\in \bar I^\varepsilon_k} \sum_{i\in I_j}\sigma_i \left(\bar{f}_\theta(C_j)-\bar{f}_\theta(\hat C_k)+\bar{f}_\theta(\hat C_k)\right)$$
where, similarly, $\bar I^\varepsilon_k=\{j \in[p]: C_j\in B_\varepsilon(\hat C_k) \}$. 
Hence, \begin{align}\label{ref:eq_quotient}
    \mathcal R_S(\mathcal F_{\mathcal A})&=\frac{1}{m}
\mathbb E_\sigma
\sup_{\bar{f}_\theta}\sum_{k=1}^{\bar p_\varepsilon}\sum_{j\in \bar I^\varepsilon_k} \sum_{i\in I_j}\sigma_i \left(\bar{f}_\theta(C_j)-\bar{f}_\theta(\hat C_k)+\bar{f}_\theta(\hat C_k)\right)
\end{align}
\paragraph{Upper bound.}
By the subadditivity of the supremum and the linearity of the expectation, we upper bound the complexity by splitting it into two terms, $T_1$ and $T_2$:
\begin{align}
\mathcal{R}_S(\mathcal{F}_{\mathcal{A}}) &\le \frac{1}{m} \mathbb{E}_\sigma \left[ \sup_{\bar{f}_\theta} \sum_{k=1}^{\bar{p}_\varepsilon} \sum_{j \in \bar{I}_k^\varepsilon} \sum_{i \in I_j} \sigma_i \left( \bar{f}_\theta(C_j) - \bar{f}_\theta(\hat{C}_k) \right) \right] \nonumber \\
&\quad + \frac{1}{m} \mathbb{E}_\sigma \left[ \sup_{\bar{f}_\theta} \sum_{k=1}^{\bar{p}_\varepsilon} \sum_{j \in \bar{I}_k^\varepsilon} \sum_{i \in I_j} \sigma_i \bar{f}_\theta(\hat{C}_k) \right] \\
&= T_1 + T_2. \nonumber
\end{align}
We first bound $T_1$
\begin{align*}
    T_1=& \frac{1}{m}
\mathbb E_\sigma \sup_{\bar{f}_\theta}\sum_{k=1}^{\bar p_\varepsilon}\sum_{j\in \bar I^\varepsilon_k} \sum_{i\in I_j}\sigma_i \left(\bar{f}_\theta(C_j)-\bar{f}_\theta(\hat C_k)\right)\\
\le &  \left|\frac{1}{m}
\mathbb E_\sigma \sup_{\bar{f}_\theta}\sum_{k=1}^{\bar p_\varepsilon}\sum_{j\in \bar I^\varepsilon_k} \sum_{i\in I_j}\sigma_i \left(\bar{f}_\theta(C_j)-\bar{f}_\theta(\hat C_k)\right)\right|\\
\le & \frac{1}{m}
\mathbb E_\sigma \sup_{\bar{f}_\theta}\sum_{k=1}^{\bar p_\varepsilon}\sum_{j\in \bar I^\varepsilon_k} \sum_{i\in I_j}\vert\sigma_i (\bar{f}_\theta(C_j)-\bar{f}_\theta(\hat C_k))\vert\\
\le & \frac{1}{m}
\mathbb E_\sigma \sup_{\bar{f}_\theta}\sum_{k=1}^{\bar p_\varepsilon}\sum_{j\in \bar I^\varepsilon_k} \sum_{i\in I_j}\vert\sigma_i \vert \vert\bar{f}_\theta(C_j)-\bar{f}_\theta(\hat C_k)\vert
\end{align*}
Due to the Lipschitzianity of the function $\bar{f}_\theta$ with respect to $\bar d$ (established in \Cref{lemma:quotient_lipschitz}) and because the distance between a class $C_j$ and its centroid $\hat C_k$ is at most $\varepsilon$ we get that :
\begin{align*}
T_1\le & \frac{1}{m}
\mathbb E_\sigma \sup_{\bar{f}_\theta}\sum_{k=1}^{\bar p_\varepsilon}\sum_{j\in \bar I^\varepsilon_k} \sum_{i\in I_j}\vert\sigma_i \vert \vert\bar{f}_\theta(C_j)-\bar{f}_\theta(\hat C_k)\vert\\
= & \frac{1}{m}
\mathbb E_\sigma \sum_{k=1}^{\bar p_\varepsilon}\sum_{j\in \bar I^\varepsilon_k} \sum_{i\in I_j}\vert\sigma_i \vert \lambda \varepsilon\\
=& \frac{\lambda \varepsilon}{m}
\mathbb E_\sigma \sum_{i=1}^m\vert\sigma_i \vert = \lambda \varepsilon.
\end{align*}
Therefore, the updated bound on the Rademacher complexity is:
$$\mathcal{R}_S(\mathcal{F_A})\leq \lambda \varepsilon + T_2.$$
Now we proceed bounding the second term $T_2$.\\
Let $U_k$ be the union of indices of graphs in the classes which are contained in the ball $B_\varepsilon(\hat C_k)$, formally $U_k:=\bigcup\limits_{j\in \bar I^\varepsilon_k} I_j$. Then,
$$T_2:= \frac{1}{m}
\mathbb E_\sigma \sup_{\bar{f}_\theta}\sum_{k=1}^{\bar p_\varepsilon}\sum_{j\in \bar I^\varepsilon_k} \sum_{i\in I_j}\sigma_i\bar{f}_\theta(\hat C_k) = \frac{1}{m}
\mathbb E_\sigma \sup_{\bar{f}_\theta}\sum_{k=1}^{\bar p_\varepsilon}\sum_{i\in U_k} \sigma_i\bar{f}_\theta(\hat C_k).$$
Let $Z_k:=\sum\limits_{i\in U_k} \sigma_i$. 
Given that the output of the $\bar f_\theta$ is in $[-1,1]$, then 
$Z_k\bar{f}_\theta(\hat C_k)\le |Z_k|$ for all $\bar{f}_\theta\in \bar{\mathcal{F_A}}$: 
\begin{equation*}
    \frac{1}{m}
    \mathbb E_\sigma \sup_{\bar{f}_\theta}\sum_{k=1}^{\bar p_\varepsilon}Z_k\bar{f}_\theta(\hat C_k)
    \le \frac{1}{m}
    \mathbb E_\sigma \sum_{k=1}^{\bar p_\varepsilon}|Z_k|
    \end{equation*}

By linearity of the expectation and by Jensen's inequality we get that $\mathbb{E}_\sigma |Z_k|=\mathbb{E}_\sigma \sqrt{Z_k^2}
\leq
\sqrt{\mathbb{E}_\sigma Z_k^2}$. Since the Rademacher variables are independent with $\mathbb{E}_\sigma[\sigma_i] = 0$, we have that   $\mathbb{E}_\sigma [Z_k^2] = \mathbb{E}_\sigma [\sum\limits_{i\in U_k} \sigma_i^2 ]=\mathbb{E}_\sigma \sum\limits_{i\in U_k}1 = |U_k|$. Substituting back in $T_2$ yields:
\begin{align*}
    T_2 &=\frac{1}{m}
\mathbb E_\sigma \sum_{k=1}^{\bar p_\varepsilon}|Z_k|\\
&=\frac{1}{m}
 \sum\limits_{k=1}^{\bar p_\varepsilon}\mathbb E_\sigma|Z_k|\\
&\le \frac{1}{m}
 \sum\limits_{k=1}^{\bar p_\varepsilon}\sqrt{\mathbb E_\sigma Z_k^2}\\
 &= \sum\limits_{k=1}^{\bar p_\varepsilon} \frac{\sqrt{|U_k|}}{m}.
\end{align*}
Combining the bounds for $T_1$ and $T_2$, we obtain the upper bound on the Rademacher complexity:
\begin{equation}
\mathcal{R}_S(\mathcal{F}_{\mathcal{A}}) \le \sum_{k=1}^{\bar{p}_\varepsilon} \frac{\sqrt{|U_k|}}{m} + \lambda \varepsilon.
\end{equation}

    \paragraph*{Lower Bound.}
    Starting from \ref{ref:eq_quotient}, we split the inner sum into two terms, $L_1$ and $L_2$:
    \begin{align*}
        \mathcal{R}_S(\mathcal{F_A})=& \frac{1}{m} \mathbb{E}_{\bm{\sigma}} \left[ \sup_\Theta \left(\sum_{k=1}^{\bar p_\varepsilon}\sum_{j\in \bar I^\varepsilon_k} \sum_{i\in I_j} \sigma_i (\bar{f}_\theta(C_j) - \bar{f}_\theta(\hat C_k)) + \sum_{k=1}^{\bar p_\varepsilon} \sum_{j\in \bar I^\varepsilon_k} \sum_{i \in I_j} \bar{f}_\theta(\hat{C}_k)\sigma_i\right)\right]\\
        =& \frac{1}{m} \mathbb{E}_{\bm{\sigma}} [\sup_\Theta(L_1 + L2)].
    \end{align*}
    By the subadditivity of the supremum, we have $\sup (L_1 + L_2) \geq \sup L_2 - \sup (-L_1)$. Furthermore, for any real value, $-x \leq |x|$, meaning we can bound the subtracted term as $\sup (-L_1) \leq \sup |L_1|$.
    \begin{align}
        \mathcal{R}_{S}(\mathcal{F_A}) &= \frac{1}{m} \mathbb{E}_{\bm{\sigma}} \left[ \sup \left( L_1 + L_2 \right) \right] \nonumber\\
        &\geq \frac{1}{m} \mathbb{E}_{\bm{\sigma}} \left[ \sup L_2 \right] - \frac{1}{m} \mathbb{E}_{\bm{\sigma}} \left[ \sup |L_1| \right] \nonumber\\
        &\geq - \frac{1}{m} \mathbb{E}_{\bm{\sigma}} \left[ \sup |L_1| \right] + \frac{1}{m} \mathbb{E}_{\bm{\sigma}} \left[ \sup L_2 \right]  \nonumber\\
    \end{align}
    
    From the Lipschitzianity of $\bar{f}_\theta$ and the $\varepsilon$-ball covering of each centroid, we know that $|\bar{f}_\theta(C_j) - \bar{f}_\theta(\hat C_k)| \leq \lambda \bar d(C_j, \hat{C}_k) \leq \lambda \varepsilon$. Substituting this absolute value upper bound back into the inequality with $L_1$, we obtain:
    \begin{align}
        \mathcal{R}_{S}(\mathcal{F_A}) &\geq - \frac{1}{m} \mathbb{E}_{\bm{\sigma}} \left[ \sup |L_1| \right] + \frac{1}{m} \mathbb{E}_{\bm{\sigma}} \left[ \sup L_2 \right]  \nonumber\\
        &\geq - \frac{1}{m} \mathbb{E}_{\bm{\sigma}} \left[ \sup \left| \sum_{k=1}^{\bar p_\varepsilon}\sum_{j\in \bar I^\varepsilon_k} \sum_{i\in I_j} \sigma_i (\bar{f}_\theta(C_j) - \bar{f}_\theta(\hat C_k)) \right| \right] + \frac{1}{m} \mathbb{E}_{\bm{\sigma}} \left[ \sup L_2 \right]  \nonumber\\
        &\geq - \frac{1}{m} \mathbb{E}_{\bm{\sigma}} \left[ \sup \sum_{k=1}^{\bar p_\varepsilon}\sum_{j\in \bar I^\varepsilon_k} \sum_{i\in I_j}  \lambda \varepsilon \right] + \frac{1}{m} \mathbb{E}_{\bm{\sigma}} \left[ \sup L_2 \right]  \nonumber\\
        &\geq  - \lambda \varepsilon + \frac{1}{m} \mathbb{E}_{\bm{\sigma}} \left[ \sup L_2 \right] \nonumber\\
    \end{align}
    
    If we similarly let $Z_k=\sum_{j\in \bar I^\varepsilon_k}  \sum_{i \in I_j} \sigma_i = \sum_{i \in U_k} \sigma_i $, where $U_k:=\bigcup_{j\in \bar I^\varepsilon_k} I_j$, then $L_2$ is the inner sum:  $\,\sum_{k=1}^{\bar p_\varepsilon}\bar{f}_\theta(\hat{C}_k)Z_k$.
   Assuming $\mathcal{F}_{\mathcal{A}}$ contains maximally expressive functions matching the bounding invariant, then the function that maximizes the sum is the one whose output perfectly aligns with the sign of the Rademacher noise sum, i.e. $\bar{f}_\theta(\hat{C}_k) = \text{sign}(Z_k)$. Using this and the linearity of expectation, we obtain:
    \begin{align}
        \frac{1}{m} \mathbb{E}_{\bm{\sigma}} \left[ \sup L_2 \right] &= \frac{1}{m} \mathbb{E}_{\bm{\sigma}} \left[ \sum_{k=1}^{\bar p_\varepsilon} |Z_j| \right] \nonumber\\
        &= \frac{1}{m} \mathbb{E}_{\bm{\sigma}} \left[ \sum_{k=1}^{\bar p_\varepsilon} |\sum_{i \in U_k} \sigma_i| \right]\nonumber\\
        &= \frac{1}{m} \sum_{k=1}^{\bar p_\varepsilon} \mathbb{E}_{\bm{\sigma}} \left[ |\sum_{i \in U_k} \sigma_i| \right] \nonumber\\
    \end{align}
    Using Khintchine's inequality, we can bound the expected value \cite{haagerup1981best} as:
    \begin{align}
        \mathbb{E}_{\bm{\sigma}} \left[ |\sum_{i \in U_k} \sigma_i| \right] &\geq \frac{1}{\sqrt{2}}\sqrt{\sum_{i \in U_k} 1^2} \nonumber\\
        &\geq \sqrt{\frac{|U_k|}{2}} \nonumber\\
    \end{align}
    Substituting this back into the main expression yields the final lower bound:
    \begin{align*}
        \mathcal{R}_{S}(\mathcal{F_A}) &\geq - \lambda \varepsilon + \frac{1}{m} \mathbb{E}_{\bm{\sigma}} \left[ \sup L_2 \right] \\
        &= - \lambda \varepsilon + \frac{1}{m} \sum_{k=1}^{\bar p_\varepsilon} \mathbb{E}_{\bm{\sigma}} \left[ |\sum_{i \in U_k} \sigma_i| \right] \\
        &\geq - \lambda \varepsilon + \frac{1}{m} \sum_{k=1}^{\bar p_\varepsilon}  \sqrt{\frac{|U_k|}{2}} \\
    \end{align*}

\end{proof}

\end{appendices}

\end{document}